\documentclass{article}

\usepackage{microtype}
\usepackage{graphicx}
\usepackage{subcaption}
\usepackage{booktabs} 
\usepackage{enumitem}
\usepackage{placeins}

\usepackage{algorithm}
\usepackage{algorithmic}

\usepackage{xcolor}         

\usepackage{tikz}
\usetikzlibrary{arrows.meta,positioning,fit,backgrounds,calc,shapes.geometric,decorations.pathreplacing,decorations.pathmorphing}

\definecolor{blockblue}{RGB}{41,128,185}
\definecolor{blockdarkblue}{RGB}{24,80,130}
\definecolor{gradientred}{RGB}{192,57,43}
\definecolor{gradientorange}{RGB}{230,126,34}
\definecolor{mavggreen}{RGB}{39,174,96}

\usepackage{hyperref}

\usepackage[accepted]{icml2026}

\usepackage{amsmath}
\usepackage{amssymb}
\usepackage{mathtools}
\usepackage{amsthm}

\usepackage[capitalize,noabbrev]{cleveref}

\theoremstyle{plain}
\newtheorem{theorem}{Theorem}[section]
\newtheorem{proposition}[theorem]{Proposition}
\newtheorem{lemma}[theorem]{Lemma}
\newtheorem{corollary}[theorem]{Corollary}
\theoremstyle{definition}

\newtheorem{assumption}[theorem]{Assumption}
\theoremstyle{remark}
\newtheorem{remark}[theorem]{Remark}

\usepackage[textsize=tiny]{todonotes}

\icmltitlerunning{Gradient Smoothing: Coupling Layer-wise Updates for Improved Optimization}

\begin{document}

\twocolumn[
  \icmltitle{Gradient Smoothing: Coupling Layer-wise Updates for \\Improved Optimization}



  \icmlsetsymbol{equal}{*}

  \begin{icmlauthorlist}
    \icmlauthor{Haoming Meng}{equal,yyy,vec}
    \icmlauthor{Anton Sugolov}{equal,yyy,vec}
    \icmlauthor{Vardan Papyan}{yyy,vec}
  \end{icmlauthorlist}

  \icmlaffiliation{yyy}{Department of Mathematics, University of Toronto, Toronto, Canada}
  \icmlaffiliation{vec}{Vector Institute, Toronto, Canada Code available at \url{https://github.com/sugolov/gradient-smoothing}}

  \icmlcorrespondingauthor{Haoming Meng}{haoming.meng@mail.utoronto.ca}

  \icmlkeywords{Machine Learning, ICML}

  \vskip 0.3in
]



\printAffiliationsAndNotice{\icmlEqualContribution}  

\begin{abstract}

Deep neural networks with repeated architectural blocks, such as transformers, often exhibit structured relationships across layers that emerge during training. Motivated by this observation, we introduce \emph{Depth-wise Gradient Augmentation}, a general optimization paradigm in which the update applied to each layer is obtained by transforming the collection of block-wise optimizer updates along the depth dimension. Within this framework, we study \emph{Gradient Smoothing}, a family of depth-wise smoothing methods, and instantiate it with a simple local \emph{Window Smoothing} operator. The resulting method operates directly on block-wise updates produced by arbitrary base optimizers (e.g., SGD, Adam, Muon), incurs minimal computational overhead, and is compatible with existing optimization pipelines. We evaluate Gradient Smoothing across a diverse set of architectures and training regimes, including language model pretraining, RL post-training of LLMs for reasoning, diffusion modeling, and image classification with Vision Transformers. Across these settings, Gradient Smoothing consistently improves optimization and generalization performance without modifying model architectures or training objectives. We further show that it promotes more structured representation evolution across depth, consistent with its interpretation as a structured depth-wise preconditioning method. Together, these results establish Depth-wise Gradient Augmentation as a promising framework for exploiting cross-depth structure in optimization and demonstrate Gradient Smoothing as a simple and broadly applicable instantiation.

\end{abstract}

\section{Introduction}

Many modern deep neural networks are composed of repeated architectural blocks
with shared computational structure, such as Residual Networks (ResNets)
\citep{he2015deepresiduallearningimage} and transformers
\citep{vaswani2023attentionneed}. Recent work has shown that networks of this
form often exhibit strong structural regularities across depth that emerge
during training. In particular, studies on \emph{Transformer Block Coupling}
\citep{aubry2025transformerblockcouplingcorrelation} and
\emph{Residual Alignment}
\citep{li2024residualalignmentuncoveringmechanisms} demonstrate that singular vectors of block
Jacobians and residual representations become aligned across layers,
suggesting a form of implicit coordination throughout depth. Related evidence
appears in studies of layer similarity and pruning in large language models
\citep{gromov2025unreasonableineffectivenessdeeperlayers}, as well as analyses
showing smooth evolution of information-theoretic quantities across
depth \citep{skean2024doesrepresentationmatterexploring}.


These observations suggest that repeated blocks in deep networks do not evolve independently, but instead exhibit coordinated representational and dynamical structure across depth. However, standard optimization methods typically construct and apply block-wise optimizer updates independently (aside from their dependencies through the forward and backward passes), without explicitly leveraging relationships between layers. This naturally raises the question: \emph{Can optimizer updates themselves be augmented to exploit the structure that emerges across depth?}

In this work, we introduce a general paradigm of \emph{Depth-wise Gradient Augmentation}. Rather than applying the optimizer update produced for each block in isolation, the update applied to each layer is obtained by applying a depth-wise augmentation operator to the collection of block-wise optimizer updates. This perspective encompasses a broad class of update transformations acting along the depth dimension.

As a first instantiation of this framework, we study \emph{Gradient Smoothing}, motivated by the empirical observation that neighboring layers often learn related representations and exhibit increasingly aligned optimization dynamics throughout training. We focus on a particularly simple realization, \emph{Window Smoothing}, which augments each block update by averaging optimizer updates from a local window of neighboring layers.
Window Smoothing is simple to implement, incurs minimal computational overhead, and is compatible with arbitrary base optimizers. Moreover, Gradient Smoothing admits a natural interpretation as a structured preconditioning method acting on block-structured parameter spaces, while providing a foundation for exploring richer forms of depth-wise update augmentation.

\paragraph{Contributions.}

\begin{enumerate}
     \item We introduce \emph{Depth-wise Gradient Augmentation}, a general optimization paradigm in which the collection of block-wise optimizer updates is transformed by a depth-wise augmentation operator before being applied. Within this framework, we study \emph{Gradient Smoothing}, a family of structured depth-wise smoothing methods, and instantiate it with a simple, optimizer-agnostic \emph{Window Smoothing} operator, providing a foundation for exploring more general forms of depth-wise update augmentation.

    \item We demonstrate that Gradient Smoothing consistently improves optimization and generalization across a diverse set of architectures and training regimes, including language model pretraining, RL fine-tuning of LLMs for reasoning, supervised learning with Vision Transformers, and diffusion models.

    \item We provide both empirical and theoretical evidence that Gradient Smoothing promotes more structured representation evolution across depth. Empirically, we observe increased layer-wise trajectory alignment and representation residual similarity, while theoretically we characterize how smoothing influences representational residual structure.


\end{enumerate}

\section{Background and Setup}

\subsection{Block-Structured Optimization}
We consider neural networks composed of $L$ repeated blocks connected by skip connections, where each block shares the same (or similar) functional form but has its own parameters. In addition, the model may include a (typically small) collection of parameters not associated with the repeated structure, such as input embeddings or output heads.

Formally, let $F:\mathbb{R}^d \times \mathbb{R}^p \to \mathbb{R}^d$ denote the residual mapping associated with a single block. For block $\ell$, the residual function is given by $F(\cdot;\theta_\ell)$, where $\theta_\ell \in \mathbb{R}^p$ are layer-specific parameters. Given the initial embedding $h_0$, the hidden states $(h_\ell)_{\ell=1}^{L} \subset \mathbb{R}^d$ then evolve according to
\begin{equation}
\label{eq:residual_model}
h_{\ell+1} = h_\ell + F(h_\ell;\theta_\ell),
\qquad \ell = 1,\dots,L-1.
\end{equation}
 Although we present the formulation for standard residual connections, the framework naturally extends to architectures with more general skip connection patterns. We write the full parameter vector as
\[
(\theta, \phi), \qquad 
\theta = (\theta_1, \dots, \theta_L), \quad \theta_l \in \mathbb{R}^{p},
\]
where $\phi\in\mathbb{R}^{D_{\phi}}$ collects all non-block parameters.
This setting encompasses a broad class of modern architectures, including residual networks with repeated block structure, Vision Transformers, and Transformer-based language models (by letting each $h_{\ell}$ be a sequence of vectors).
We denote the training objective by $\mathcal{L}(\theta,\phi)$ and assume it is differentiable.

At optimization step $t$, the block-wise gradients are given by
\[
g_l^{(t)} := \nabla_{\theta_l} \mathcal{L}(\theta^{(t)},\phi^{(t)}), 
\qquad 
g^{(t)} := (g_1^{(t)}, \dots, g_L^{(t)}),
\]
and we denote by $g_\phi^{(t)} := \nabla_\phi \mathcal{L}(\theta^{(t)},\phi^{(t)})$ the gradient with respect to $\phi$.

Standard first-order methods treat the block-wise gradients $\{g_l^{(t)}\}_{l=1}^L$ independently.
However, repeated-block architectures induce strong structural relationships across depth: adjacent blocks often learn similar transformations and exhibit aligned representations, particularly after an initial transient phase of training.
This suggests that the stacked gradient vector $g^{(t)}$ possesses meaningful structure along the depth dimension.

\begin{figure}[h!]
\centering
\resizebox{\columnwidth}{!}{%
\begin{tikzpicture}[
	node distance=0.6cm and 1.5cm,
	scale=0.85, transform shape,
	block/.style={draw, fill=blockblue!40, minimum width=1.6cm, minimum height=0.8cm, align=center, rounded corners},
	gradient/.style={draw=gradientred!80, fill=gradientred!30, align=center, minimum width=2.2cm, minimum height=0.8cm, rounded corners},
	filter/.style={draw=mavggreen!80, fill=mavggreen!30, align=center, minimum width=2cm, minimum height=3cm, rounded corners},
	output/.style={draw=gradientorange!80, fill=gradientorange!30, align=center, minimum width=2.2cm, minimum height=0.8cm, rounded corners},
	arrow/.style={-{Stealth[length=3mm, width=2mm]}, line width=0.8pt},
	thick-arrow/.style={-{Stealth[length=4mm, width=2.5mm]}, line width=1pt},
	block-arrow/.style={-{Stealth[length=3.5mm, width=2.2mm]}, line width=1pt, color=blockdarkblue},
	bracket/.style={decorate, decoration={brace, amplitude=10pt, raise=2pt}, thick},
]

\node[block] (blockL) at (-2,3) {$f_L$};
\node[block] (block3) at (-2,0) {$f_3$};
\node[block] (block2) at (-2,-1.5) {$f_2$};
\node[block] (block1) at (-2,-3) {$f_1$};

\node at (-2,1.5) {$\vdots$};

\node[gradient, right=of blockL] (gradL)
  {$\mathcal{U}\!\left(\nabla_{\theta_L} L\right)$};
\node[gradient, right=of block3] (grad3)
  {$\mathcal{U}\!\left(\nabla_{\theta_3} L\right)$};
\node[gradient, right=of block2] (grad2)
  {$\mathcal{U}\!\left(\nabla_{\theta_2} L\right)$};
\node[gradient, right=of block1] (grad1)
  {$\mathcal{U}\!\left(\nabla_{\theta_1} L\right)$};

\node at (1.5,1.5) {$\vdots$};

\draw[arrow] (blockL) -- (gradL);
\draw[arrow] (block3) -- (grad3);
\draw[arrow] (block2) -- (grad2);
\draw[arrow] (block1) -- (grad1);

\draw[arrow] (block1) -- (block2);
\draw[arrow] (block2) -- (block3);
\draw[arrow] (block3) -- ($(block3)+(0,1)$);
\draw[arrow] ($(block3)+(0,2)$) -- (blockL);

\node[filter, right=2cm of grad3] (filter) {\\\\\\Filter $S$};

\node[output, right=6cm of gradL] (filtgradL)
  {$\tilde{\mathcal{U}}\!\left(\nabla_{\theta_L} L\right)$};
\node[output, right=6cm of grad3] (filtgrad3)
  {$\tilde{\mathcal{U}}\!\left(\nabla_{\theta_3} L\right)$};
\node[output, right=6cm of grad2] (filtgrad2)
  {$\tilde{\mathcal{U}}\!\left(\nabla_{\theta_2} L\right)$};
\node[output, right=6cm of grad1] (filtgrad1)
  {$\tilde{\mathcal{U}}\!\left(\nabla_{\theta_1} L\right)$};

\node at (9.5,1.5) {$\vdots$};

\begin{scope}[shift={($(filter.north)+(0,-1.3)$)}, scale=0.65]
	\fill[black!80] (-0.15,0) rectangle (0.15,0.7);

	\fill[black!80] (0.15,0) rectangle (0.45,0.7*0.606);
	\fill[black!80] (0.45,0) rectangle (0.75,0.7*0.368);
	\fill[black!80] (0.75,0) rectangle (1.05,0.7*0.223);
	\fill[black!80] (1.05,0) rectangle (1.35,0.7*0.135);

	\fill[black!80] (-0.45,0) rectangle (-0.15,0.7*0.606);
	\fill[black!80] (-0.75,0) rectangle (-0.45,0.7*0.368);
	\fill[black!80] (-1.05,0) rectangle (-0.75,0.7*0.223);
	\fill[black!80] (-1.35,0) rectangle (-1.05,0.7*0.135);

	\draw (-1.5,0) -- (1.5,0);
\end{scope}

\draw[bracket] ($(gradL.east)+(0.2,0.5)$) -- ($(grad1.east)+(0.2,-0.5)$);
\draw[bracket, decoration={brace, mirror, amplitude=10pt, raise=2pt}]
($(filtgradL.west)+(-0.2,0.5)$) -- ($(filtgrad1.west)+(-0.2,-0.5)$);

\draw[thick-arrow] ($(grad3.east)+(0.8,0)$) -- ($(filter.west)+(-0.2,0)$);
\draw[thick-arrow] ($(filter.east)+(0.2,0)$) -- ($(filtgrad3.west)+(-0.8,0)$);

\end{tikzpicture}
}
\caption{ {\bf Gradient Smoothing.} 
Representation of the gradient augmentation scheme applied across depth to the updates in a deep network after backpropagation. For a deep network with $L$ identical architectural blocks (but with different parameters), the gradient updates in each parameter block $\theta^l$ are reweighted across depth to stabilize information propagation.
}
\end{figure}
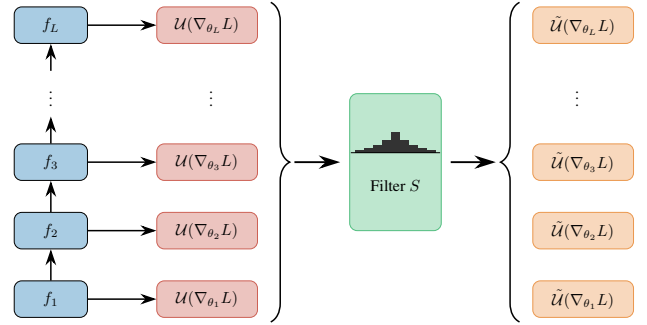

\begin{figure*}[h]
    \centering




        \begin{subfigure}[t]{0.32\linewidth}
        \centering
        \includegraphics[width=\linewidth]{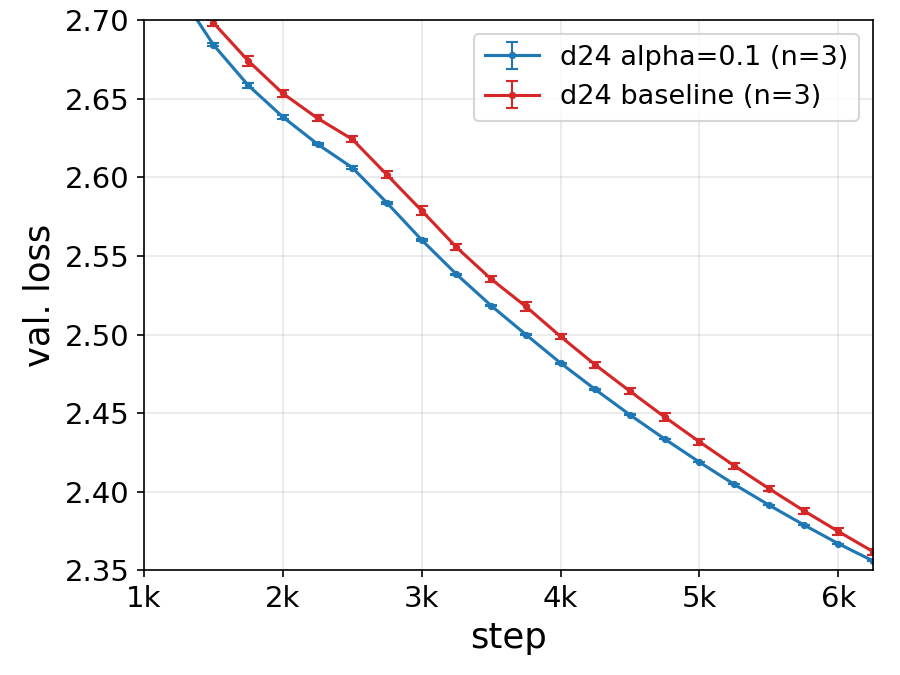}
        \caption{Depth $24$: Validation Loss.}
        \label{fig:nanochat_d24_val}
    \end{subfigure}
    \hfill
    \begin{subfigure}[t]{0.32\linewidth}
        \centering
        \includegraphics[width=\linewidth]{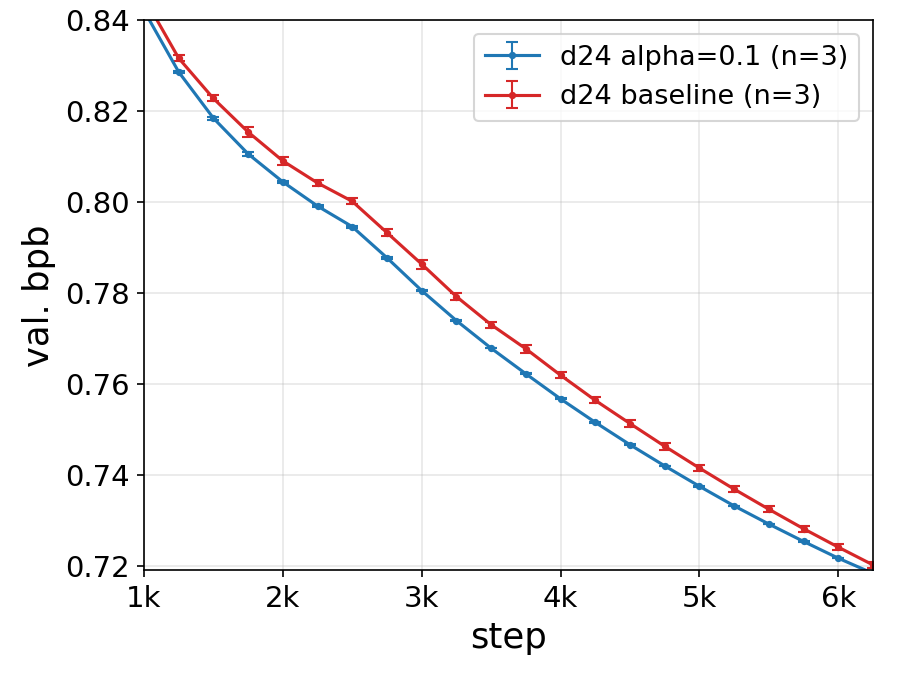}
        \caption{Depth $24$: Validation BPB.}
        \label{fig:nanochat_d24_bpb}
    \end{subfigure}
    \hfill
    \begin{subfigure}[t]{0.32\linewidth}
        \centering
        \includegraphics[width=\linewidth]{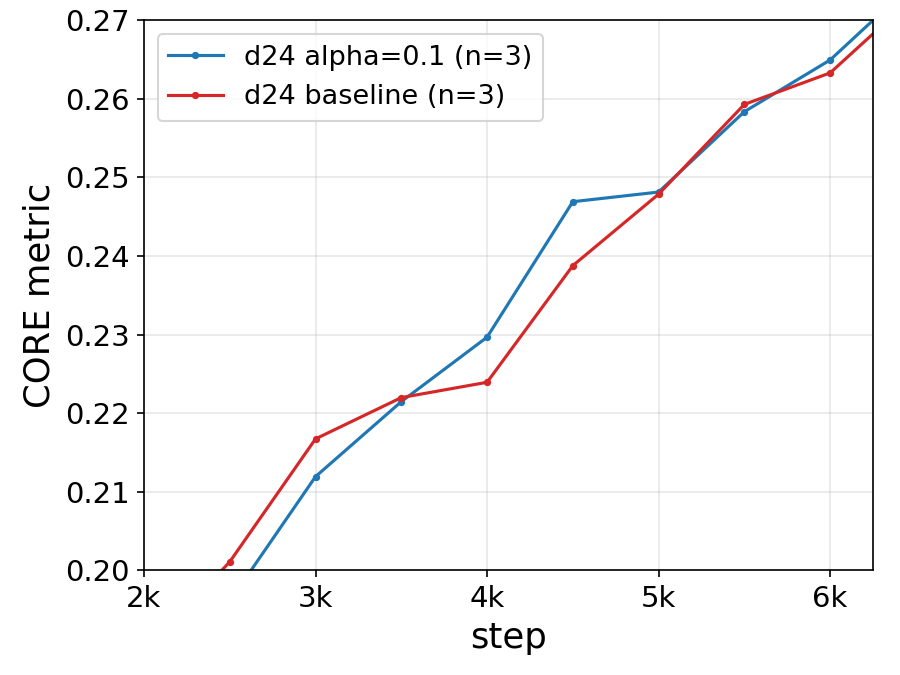}
        \caption{Depth $24$: CORE Metric.}
        \label{fig:nanochat_d24_core}
    \end{subfigure}

    \vspace{2mm}

    \begin{subfigure}[t]{0.32\linewidth}
        \centering
        \includegraphics[width=\linewidth]{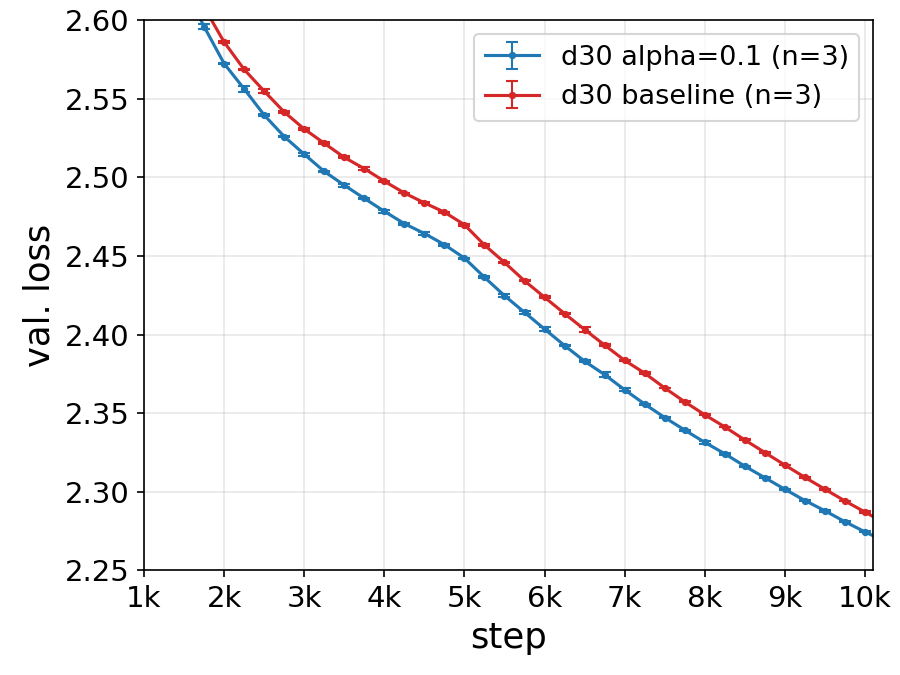}
        \caption{Depth $30$: Validation Loss.}
        \label{fig:nanochat_d30_val}
    \end{subfigure}
    \hfill
    \begin{subfigure}[t]{0.32\linewidth}
        \centering
        \includegraphics[width=\linewidth]{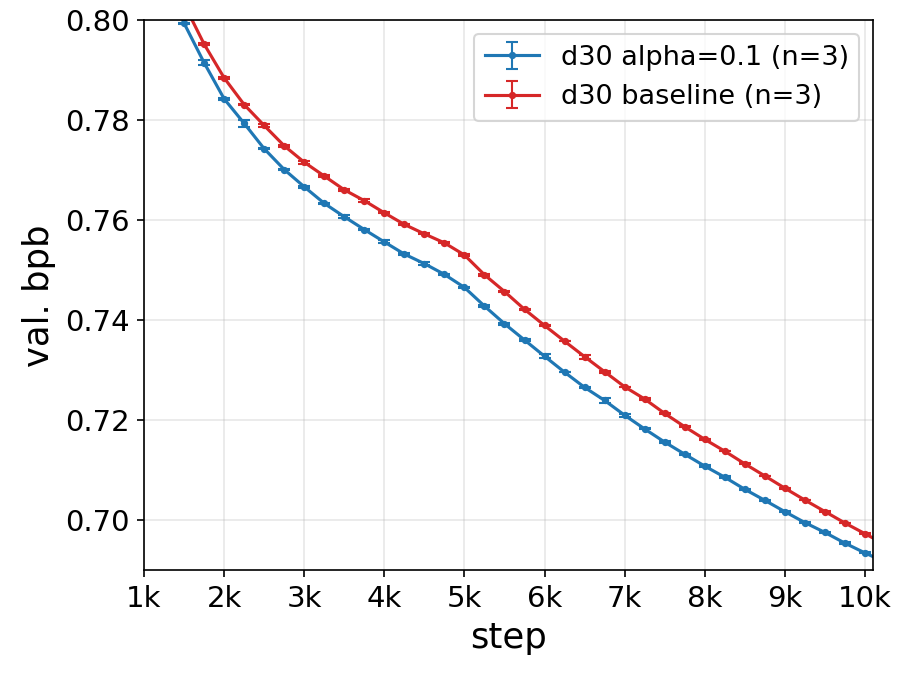}
        \caption{Depth $30$: Validation BPB.}
        \label{fig:nanochat_d30_bpb}
    \end{subfigure}
    \hfill
    \begin{subfigure}[t]{0.32\linewidth}
        \centering
        \includegraphics[width=\linewidth]{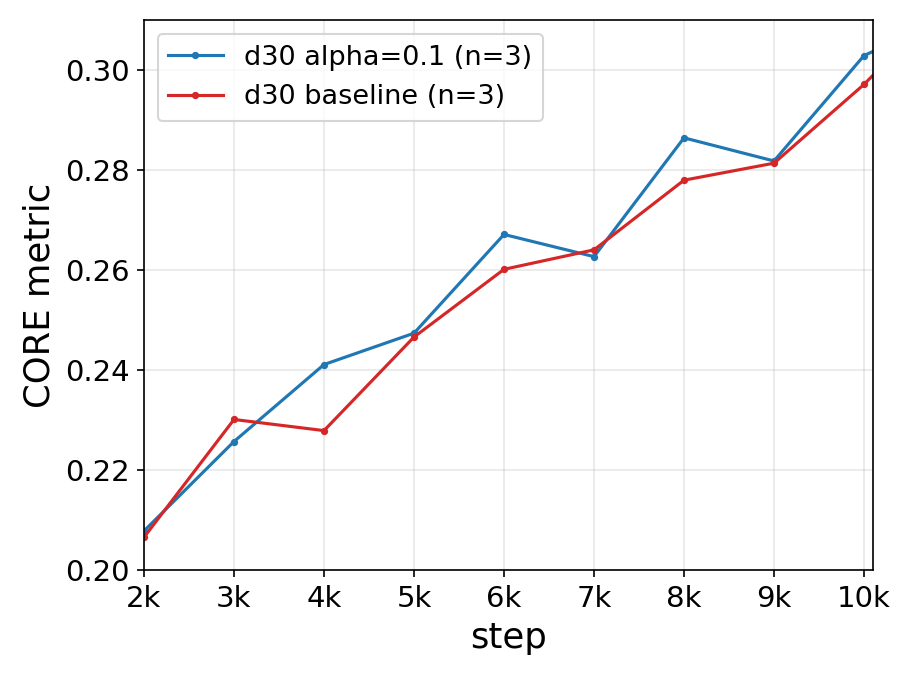}
        \caption{Depth $30$: CORE Metric.}
        \label{fig:nanochat_d30_core}
    \end{subfigure}

    \vspace{-1mm}
    \caption{{\bf Nanochat pretraining with Gradient Smoothing.}
    Validation loss, BPB, and DCLM CORE metric during \texttt{nanochat} pretraining of GPT2 with the default Adam\,+\,NorMuon optimizer setup (Table \ref{tab:nanochat-config}). We compare the baseline against a run with identical hyperparameters using depth-wise update smoothing applied to the transformer blocks. Models of depth 24 (1.38B parameters, top row) and depth 30 (2.40B parameters bottom row) with $\alpha = 0.1$ are replicated against the baseline for $\textbf{seed} \in \{42, 67, 2026\}$.
    At both depths, smoothing accelerates convergence of validation loss and improves the CORE metric during training, with a larger gap visible at depth $30$. } 
    \label{fig:nanochat_combined}
\end{figure*}

\subsection{General First-Order Updates}

Rather than committing to a specific optimizer, we consider a general first-order update rule of the form
\[
(\theta^{(t+1)},\phi^{(t+1)}) 
= (\theta^{(t)},\phi^{(t)}) - \eta\, (u^{(t)}, u_\phi^{(t)}),
\]
where
\[
u^{(t)} := \mathcal{U}^{(t)}(g^{(t)}), 
\qquad 
u_\phi^{(t)} := \mathcal{U}_\phi^{(t)}(g_\phi^{(t)})
\]
denote the updates produced by the base optimizer for the block and non-block parameters, respectively.

This formulation includes gradient descent ($u^{(t)} = g^{(t)}$), momentum methods, adaptive optimizers such as Adam or AdamW, as well as Muon and related variants. Importantly, $u^{(t)}$ may depend on both the current gradient and internal optimizer state, and may involve nonlinear or time-varying transformations of $g^{(t)}$. 

This optimizer-agnostic viewpoint allows us to study update augmentation directly at the level of optimizer updates, independent of the specific mechanisms used to construct them. As a result, the framework can be incorporated as a lightweight, modular extension to a broad class of existing optimizers.


\section{Depth-wise Gradient Augmentation}

\subsection{General Framework}
\label{subsec:coupling_depth}

Networks built from repeated architectural blocks develop strong cross-depth structure in their representations during training. Motivated by this observation, we introduce \emph{Depth-wise Gradient Augmentation}, a general optimization paradigm in which the update applied to each layer is obtained by transforming the collection of block-wise optimizer updates along the depth axis. This perspective provides a unified framework for incorporating cross-depth information into optimizer updates while remaining agnostic to the particular transformation used.


At iteration $t$, let
\[
u_\ell^{(t)} \in \mathbb{R}^{p},
\qquad \ell=1,\ldots,L,
\]
denote the block-wise update produced by a base optimizer
(e.g., SGD, Adam, Muon) for the parameters of layer $\ell$, and write
\[
u^{(t)} = (u_1^{(t)}, \dots, u_L^{(t)})
\in (\mathbb{R}^{p})^{L},
\]
together with the corresponding update $u_\phi^{(t)}$ for the non-block parameters.

In its most general form, Depth-wise Gradient Augmentation applies an operator
along the depth dimension independently for each within-block coordinate. For
each coordinate $i=1,\ldots,p$, let
\[
\mathcal{C}^{(i)}:\mathbb{R}^{L}\rightarrow\mathbb{R}^{L},
\]
be a depth-wise augmentation operator. Then
\[
\bigl(\tilde u_{1}^{(i)},\ldots,\tilde u_{L}^{(i)}\bigr)
=
\mathcal{C}^{(i)}
\bigl(u_{1}^{(i)},\ldots,u_{L}^{(i)}\bigr).
\]
Equivalently,
\[
\tilde u_{\ell}^{(i)}
=
\mathcal{C}^{(i)}_{\ell}
\bigl(u_{1}^{(i)},\ldots,u_{L}^{(i)}\bigr),
\qquad
\ell=1,\ldots,L.
\]
Thus, each coordinate of the update for layer $\ell$ may depend only on the corresponding coordinate across layers, while remaining independent of all other within-block coordinates.


Among the many possible choices of depth-wise augmentation operators, we focus on operators that smooth optimizer updates across depth.



\subsection{Gradient Smoothing}

A natural class of Depth-wise Gradient Augmentation methods is obtained by choosing the coordinate-wise depth augmentation operators to be a common linear smoothing operator. We refer to this family of methods as \emph{Gradient Smoothing}. Specifically, we take
\[
\mathcal{C}^{(i)} = S,
\qquad i=1,\ldots,p,
\]
where
$
S \in \mathbb{R}^{L \times L}
$
is a linear operator acting on the layer index. Equivalently, Gradient Smoothing is represented by the block-structured operator
\[
P := S \otimes I
\in \mathbb{R}^{D_\theta \times D_\theta},
\]
where $I$ denotes the identity operator on the within-block parameter space
$\mathbb{R}^p$.

The transformed block-wise update is then given by
\[
\tilde u^{(t)}
:=
P\,u^{(t)}
=
(S \otimes I)\,u^{(t)}.
\]

To describe the update on the full parameter vector $(\theta,\phi)$, we lift
$P$ to the block-diagonal operator
\[
\bar P
:=
\begin{pmatrix}
P & 0 \\
0 & I_\phi
\end{pmatrix}
\in
\mathbb{R}^{(D_\theta + D_\phi)\times(D_\theta + D_\phi)},
\]
where $I_\phi$ denotes the identity operator on the $\phi$-coordinates.

Given the full optimizer update
\[
u_{\mathrm{full}}^{(t)}
:=
\bigl(u^{(t)},\,u_\phi^{(t)}\bigr),
\]
Gradient Smoothing applies $\bar P$ to obtain
\[
\tilde u_{\mathrm{full}}^{(t)}
=
\bar P\,u_{\mathrm{full}}^{(t)}
=
\bigl(Pu^{(t)},\,u_\phi^{(t)}\bigr).
\]

The resulting optimization step becomes
\[
\theta^{(t+1)}
=
\theta^{(t)}
-
\eta\,\tilde u^{(t)},
\qquad
\phi^{(t+1)}
=
\phi^{(t)}
-
\eta\,u_\phi^{(t)}.
\]

Gradient Smoothing thus corresponds to a structured depth-wise coupling of
optimizer updates, or equivalently, a structured preconditioning operator acting
on the block-wise updates. Importantly, this construction is independent
of the specific choice of base optimizer and operates purely at the level of
update vectors.

\subsection{Smoothing Operator}
\label{sec:smooth_operators}

In this work, we instantiate Gradient Smoothing with a simple local averaging operator, which we term \emph{Window Smoothing}. Taking $S$ to be a local window averaging operator across depth, the smoothed update is given by
\begin{equation}
\tilde u_l =
\begin{cases}
\big(1-\tfrac{\alpha}{2}\big)u_1 + \tfrac{\alpha}{2}u_2, & l=1,\\[4pt]
(1-\alpha)u_l + \tfrac{\alpha}{2}(u_{l-1}+u_{l+1}), & 2\le l\le L-1,\\[4pt]
\big(1-\tfrac{\alpha}{2}\big)u_L + \tfrac{\alpha}{2}u_{L-1}, & l=L,
\end{cases}
\label{eq:local_avg}
\end{equation}
where $\alpha\in[0,1)$ controls the smoothing strength. Each block's update is replaced by a convex combination of its own update and those of its immediate neighbors in depth; the boundary cases at $l=1$ and $l=L$ use a half-weight on the single available neighbor so that $S$ remains row-stochastic ($S\mathbf{1}=\mathbf{1}$). The resulting $S\in\mathbb{R}^{L\times L}$ is a symmetric tridiagonal matrix and introduces no additional trainable parameters.

Many other depth-wise weighting schemes are possible within this framework. For instance, exponential smoothing, higher-order tridiagonal kernels, and local window averaging with larger window sizes all fit the same general paradigm. Our work focuses on local window averaging as a simple yet effective instantiation of Gradient Smoothing, applicable to many optimizers and across broad tasks.

\begin{algorithm}[tb]
   \caption{Gradient Smoothing (window, strength $\alpha$)}
   \label{alg:gradient-smoothing}
\begin{algorithmic}
   \STATE {\bfseries Input:} block weights $\{\boldsymbol{\theta}_\ell\}_{\ell=1}^{L}$,
      non-block weights $\boldsymbol{\phi}$, loss $L$,
      base optimizer $\mathcal{U}$,
      learning rate $\eta$, weight decay $\lambda$,
      smoothing strength $\alpha \in [0,1)$
   \FOR{$t=1, 2, \ldots$}
      \STATE $\mathbf{u}_\ell \gets \mathcal{U}_\theta\!\left(\nabla_{\boldsymbol{\theta}_\ell} L\right)$
         \hfill $\ell = 1,\ldots,L$
      \STATE $\mathbf{u}_\phi \gets \mathcal{U}_\phi\!\left(\nabla_{\boldsymbol{\phi}} L\right)$
      \STATE {\color{teal} $\mathbf{u}_\ell \gets
         \dfrac{\mathbf{u}_\ell}{\|\mathbf{u}_\ell\|_2}\,$
         \hfill optional normalization}
      \STATE {\color{blue} $\tilde{\mathbf{u}}_1 \gets
         \left(1-\tfrac{\alpha}{2}\right)\mathbf{u}_1
         + \tfrac{\alpha}{2}\,\mathbf{u}_2$}
      \STATE {\color{blue} $\tilde{\mathbf{u}}_\ell \gets
         (1-\alpha)\,\mathbf{u}_\ell
         + \tfrac{\alpha}{2}\!\left(\mathbf{u}_{\ell-1}+\mathbf{u}_{\ell+1}\right)$
         \hfill $\ell = 2,\ldots,L{-}1$}
      \STATE {\color{blue} $\tilde{\mathbf{u}}_L \gets
         \left(1-\tfrac{\alpha}{2}\right)\mathbf{u}_L
         + \tfrac{\alpha}{2}\,\mathbf{u}_{L-1}$}
      \STATE {\color{violet} $\tilde{\mathbf{u}}_\ell \gets
         \dfrac{\|\mathbf{u}_\ell\|_2}{\|\tilde{\mathbf{u}}_\ell\|_2}\,
         \tilde{\mathbf{u}}_\ell$
         \hfill optional norm pres.}
      \STATE $\boldsymbol{\theta}_\ell \gets \boldsymbol{\theta}_\ell
         {  \color{red} -\eta\lambda\,\boldsymbol{\theta}_\ell}  - \eta\,\tilde{\mathbf{u}}_\ell$
         \hfill {\color{red} optional w.d. decouple} 
      \STATE $\boldsymbol{\phi} \gets \boldsymbol{\phi} - \eta\,\mathbf{u}_\phi$
   \ENDFOR
\end{algorithmic}
\end{algorithm}

\subsection{Normalization Variants}

We additionally consider two normalization variants that separate the effects of smoothing on update directions and update magnitudes.

\begin{enumerate}[label=\roman*)]
    \item \textbf{Norm-Preserving Smoothing.}
    After applying the smoothing operator, each block update is rescaled to match the norm of the original optimizer update:
    \[
    \tilde u_\ell
    =
    \frac{\|u_\ell\|_2}
         {\|(Pu)_\ell\|_2}
    (Pu)_\ell.
    \]
    This preserves the magnitude of each block update while allowing smoothing to modify only its direction.

    \item \textbf{Directional Smoothing.}
    Before smoothing, each block update is normalized to unit norm, so that smoothing depends only on update directions. The resulting smoothed direction is then rescaled to the original update norm:
    \[
    \tilde u_\ell
    =
    \frac{\|u_\ell\|_2\,
    \bigl(P(u_1/\|u_1\|_2,\ldots,u_L/\|u_L\|_2)\bigr)_\ell}
    {\left\|
    \bigl(P(u_1/\|u_1\|_2,\ldots,u_L/\|u_L\|_2)\bigr)_\ell
    \right\|_2}.
    \]
    This variant performs smoothing entirely in the space of update directions while preserving the original update magnitude.
\end{enumerate}
Throughout the paper, we distinguish three normalization variants: \emph{Standard} (no normalization), \emph{Norm} (Norm-Preserving Smoothing), and \emph{Dir} (Directional Smoothing).

\section{Experiments}

We evaluate Gradient Smoothing across four distinct training regimes that stress different aspects of optimization: (i) RL fine-tuning of LLMs on reasoning tasks, (ii) language model pretraining, (iii) supervised image classification with Vision Transformers, and (iv) diffusion model training. To demonstrate compatibility with existing training setups, in all settings we compare against standard tuned baselines while keeping the underlying training \emph{hyperparameters unchanged when applying smoothing}. All base-optimizer hyperparameters are inherited unchanged from the tuned baselines, suggesting that further gains may be possible with dedicated tuning for smoothing. Because the method operates solely on parameter updates, it requires no architectural modifications and introduces negligible computational overhead, making it easy to integrate across a wide range of models and training regimes.

\subsection{Experimental Setup}

\paragraph{Gradient Smoothing Configuration.}
Gradient smoothing is applied across block-wise updates at each optimization step (Section \ref{sec:smooth_operators}) with $\alpha$-window averaging across depth. 
In addition to experimenting with $\alpha$-window settings, we test standard, norm-preserving, and directional variants for update smoothing. 

We also consider varied parameter groups within each block that are affected by any update smoothing operations.   
In Transformers/ViTs, most capacity and gradient signal flows through linear projections (attention $W_Q,W_K,W_V,W_O$ and MLP layers),
while normalization and embedding parameters are lower-dimensional and often behave differently under optimization.
Specifically we distinguish between
(i) smoothing applied to \emph{all} block parameters (which we refer to as \textit{Full} in Section~\ref{subsec:exp_results}),
(ii) smoothing applied \emph{only} to linear layers within each block (the non-normalization layers in the models we use) which we call \textit{Proj}.

Gradient Smoothing follows the weight decay convention of the underlying base optimizer. For optimizers with decoupled weight decay (e.g., AdamW), smoothing is applied only to the update produced by the loss gradients, with the weight decay term applied separately.

\begin{figure}
    \centering
    \includegraphics[width=0.9\linewidth]{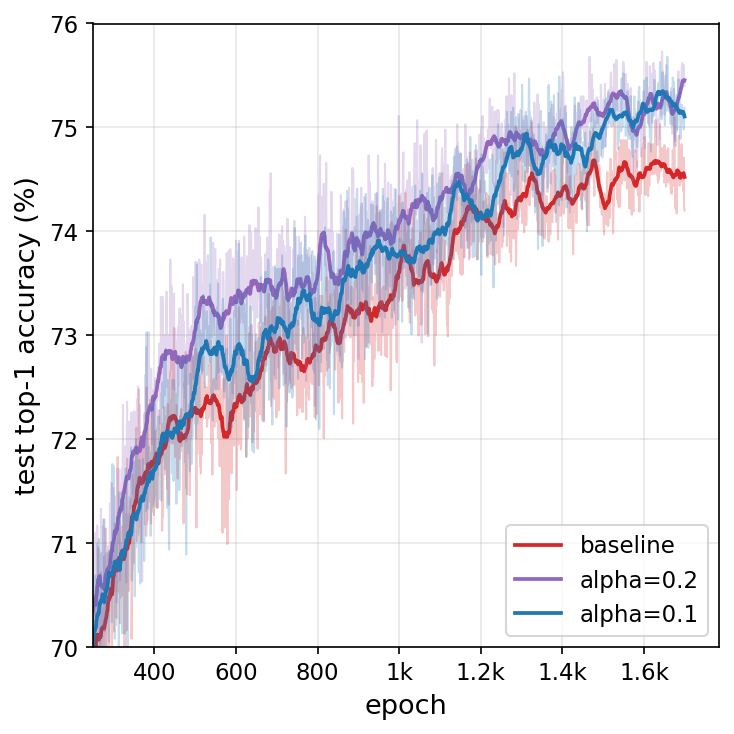}
    \caption{{\bf Test accuracy improvement with smoothing.} Test accuracy of ViT-B trained on CIFAR-100 with DeiT training recipe with data augmentations for 1700 epochs. We compare baseline training ($\alpha = 0$, red) against window smoothing with $\alpha = 0.1$ (blue) and $\alpha = 0.2$ (purple). Solid lines show running average trendlines (25 epochs) for clarity. Both smoothing configurations outperform baseline, with $\alpha = 0.2$ achieving a final test accuracy of 75.62\% compared to the baseline 74.56\% (Table~\ref{tab:vitb_cifar100}), and suggests that moderate gradient smoothing provides a consistent performance benefit throughout training.}
    \label{fig:test_acc_cifar100}
\end{figure}

\paragraph{RLVR for LLM Reasoning:} The base model we use is \texttt{DeepSeek-R1-Distill-Qwen-1.5B} \citep{Guo_2025, yang2024qwen2technicalreport}. We train with Group Relative Policy Optimization (GRPO) \citep{shao2024deepseekmathpushinglimitsmathematical} on the dataset from Open-RS \citep{dang2026reinforcementlearningreasoningsmall} (without the cosine response length reward), and then evaluate on downstream mathematical reasoning tasks: AIME24, AIME25, AMC23, MATH-500. We train the model for up to 200 optimization steps with AdamW using a base learning rate of $10^{-6}$ with a cosine learning-rate schedule and a minimum learning-rate ratio of $0.1$, including a linear warmup over the first $10\%$ of training steps. Evaluation sampling is done with temperature $0.6$ and top\_p $0.95$.

\paragraph{LLM Pretraining:}
We further evaluate gradient smoothing in a language-model pretraining setting using the
\texttt{nanochat} recipe \citep{nanochat}. \texttt{nanochat} provides a compact
end-to-end GPT-style pretraining pipeline in which the transformer depth is the main scale
parameter, while model width, number of heads, learning-rate scaling, training horizon,
batch size, and weight decay are automatically determined from this depth. We train
decoder-only transformer models with depths $24$ and $30$ (approximately $1.38$B and
$2.40$B parameters respectively, Table~\ref{tab:nanochat-config}) using the default
\texttt{nanochat} architecture and Muon-style \citep{jordan2024muon} optimization setup with NorMuon \citep{li2025normuonmakingmuonefficient}. Both depths use the same
auto-selected global batch size of $\approx1$M tokens, and step count is set by the default target parameter-data ratio. For each depth, we
compare the baseline run with an otherwise identical run where depth-wise update smoothing is applied to the repeated transformer blocks. We keep all other training settings fixed
and evaluate using validation loss and the DCLM CORE metric from the
\texttt{nanochat} evaluation suite.

\paragraph{Vision Transformer Image Classification:} ViT-B/4 models following DeiT training \citep{touvron2021trainingdataefficientimagetransformers} with default image augmentations. We train for 1700 epochs using AdamW with a cosine learning rate schedule, initial learning rate of $5 \times 10^{-4}$, and weight decay of 0.05. CIFAR-100 evaluated via top-1 accuracy.

\begin{table}[t!]
\centering
\caption{\texttt{nanochat} pretraining configurations. Width ($d_{\text{model}}$),
number of heads, and the global batch size are auto-derived from the transformer
depth by the scaling configuration.}
\label{tab:nanochat-config}
\begin{tabular}{lrrrrr}
\toprule
Variant & Layers & $d_{\text{model}}$ & Params & Batch (tokens) \\
\midrule
$d{=}24$ & 24 & 1536 & 1.38\,B & 1{,}048{,}576 \\
$d{=}30$ & 30 & 1920 & 2.40\,B & 1{,}048{,}576 \\
\bottomrule
\end{tabular}
\end{table}

\paragraph{Diffusion:} We evaluate gradient smoothing on U-ViT~\citep{bao2023worthwordsvitbackbone}, a ViT-based diffusion backbone with long skip connections proposed for image generation on CIFAR-10. The U-ViT architecture has $12$ transformer blocks with embedding dimension $512$ and $8$ attention heads. Optimization uses AdamW with learning rate $2\times 10^{-4}$, weight decay $0.03$, batch size $128$, with a $2500$-step warmup for $500{,}000$ steps. We sample via Euler--Maruyama SDE at $50$ NFE and report FID against both $10$k and $50$k generated samples.

\subsection{Experimental Results}
\label{subsec:exp_results}
Gradient Smoothing consistently improves performance across the diverse set of training regimes and benchmarks we evaluate. Notably, even simple local Window Smoothing is effective across domains, suggesting that the benefits do not depend on carefully engineered or highly specialized smoothing operators, though more sophisticated operators may yield further gains. Importantly, in all experiments, smoothing is applied directly on top of already tuned baseline training setups without modifying the underlying training hyperparameters, yet consistently improves performance without any additional hyperparameter tuning specific to Gradient Smoothing.

\begin{table*}[t]
\centering
\caption{LLM mathematical reasoning performance (\texttt{pass@1} accuracy) after RLVR fine-tuning. Gradient Smoothing improves accuracy on downstream mathematical benchmarks compared to fine-tuning just with AdamW. Results are averaged across 4 runs.}
\label{tab:llm_reasoning}

\begin{tabular}{lccccc}
\toprule
\textbf{Method} & \textbf{AIME24} & \textbf{AIME25} & \textbf{AMC23} & \textbf{MATH-500} & \textbf{Avg} \\
\midrule
Base model & 26.67 & 16.67 & 70.00 & 83.00 & 49.09\\
\midrule
Baseline (AdamW) & 35.00 & 24.17 & 78.12 & 84.50 & 55.45 {\scriptsize$\pm$ 0.99} \\
+ Window ($\alpha=0.1$, Standard, Full) & 35.83 & \underline{28.33} & \underline{78.75} & 84.30 & 56.80 {\scriptsize$\pm$ 0.21} \\
+ Window ($\alpha=0.1$, Standard, Proj) & \textbf{40.83} & 25.00 & \textbf{80.00} & \underline{84.55} & \textbf{57.60} {\scriptsize$\pm$ 1.12} \\
+ Window ($\alpha=0.2$, Standard, Proj) & \underline{36.67} & \textbf{30.00} & 77.50 & \textbf{85.35} & \underline{57.38} {\scriptsize$\pm$ 0.96} \\
\bottomrule
\end{tabular}
\end{table*}

\paragraph{RLVR Fine-Tuning for LLM Reasoning:}
Gradient Smoothing improves the reasoning accuracy of \texttt{DeepSeek-R1-Distill-Qwen-1.5B} after RLVR fine-tuning across mathematical benchmarks (Table~\ref{tab:llm_reasoning}). All Window Smoothing configurations outperform the AdamW baseline in average \texttt{pass@1} accuracy, with the best configuration achieving $57.60\%$ versus $55.45\%$. Across individual benchmarks, Gradient Smoothing yields substantial improvements, including gains on AIME24 (40.83\% vs.\ 35.00\%), AIME25 (30.00\% vs.\ 24.17\%), AMC23 (80.00\% vs.\ 78.12\%), and MATH-500 (85.35\% vs.\ 84.50\%) under different smoothing configurations. Notably, these improvements are achieved without modifying the RL objective or reward structure, suggesting that regularizing the optimization dynamics through depth-wise update smoothing can improve RL-based reasoning fine-tuning.

\paragraph{LLM Pretraining:}
Gradient Smoothing also improves the optimization trajectory in language-model pretraining with the \texttt{nanochat} recipe (Figure~\ref{fig:nanochat_combined}). During training, across both depth-$24$ and depth-$30$ models, smoothing consistently accelerates convergence of the validation loss/BPB while simultaneously improving the CORE metric. The improvements emerge at early steps and persist across the training trajectory, suggesting that smoothing provides a useful inductive bias for the model. Notably, these gains are obtained in an optimized pretraining setup using the NorMuon/AdamW optimization pipeline from \texttt{nanochat}, indicating that the benefits of smoothing extend beyond standard Adam-family training alone. The stronger improvements observed for the deeper depth-$30$ model further suggest that coupling updates across depth may become increasingly beneficial as the model depth increases.

\paragraph{Vision Transformer Image Classification:}

\begin{table}[t]
\centering
\caption{Image classification accuracy (\%) for ViT-B on CIFAR-100.}
\label{tab:vitb_cifar100}
\begin{tabular}{lc}
\toprule
\textbf{Method} & \textbf{CIFAR-100} \\
\midrule
Baseline (AdamW) & 74.56 \\
+ Window ($\alpha=0.1$, Norm, Proj) & \underline{75.44} \\
+ Window ($\alpha=0.2$, Dir, Proj) & \textbf{75.62} \\
\bottomrule
\end{tabular}
\end{table}

Gradient Smoothing improves test accuracy of ViT-B on CIFAR-100 (Table \ref{tab:vitb_cifar100}, Figure \ref{fig:test_acc_cifar100}). Window smoothing with $\alpha = 0.2$ achieves notable improvements (75.62\% vs 74.56\% baseline), while $\alpha = 0.1$ also improves over baseline (75.44\%). These results suggest that the smoothing regime provides consistent regularization for vision transformer training, with gains emerging early and persisting throughout optimization.

\paragraph{Diffusion:}
We evaluate Gradient Smoothing on U-ViT~\citep{bao2023worthwordsvitbackbone}, a ViT-based diffusion backbone with long U-Net-style skip connections which provides a strong tuned baseline on CIFAR-10. With 50 NFE sampling, Gradient Smoothing reduces FID@10k from $6.58$ to $5.82$ and FID@50k from $4.01$ to $3.74$ (Table~\ref{tab:fid_cifar10}). All runs share identical architecture, hyperparameters, and sampling budget as the tuned baseline. These results demonstrate that the benefits of Gradient Smoothing extend to generative diffusion modeling, and beyond architectures with standard residual connections to transformer backbones with long U-Net-style skip connections.

\begin{table}[t]
\centering
\caption{FID scores on CIFAR-10 for U-ViT diffusion training with and without Gradient Smoothing.}
\label{tab:fid_cifar10}
\begin{tabular}{lcc}
\toprule
\textbf{Method} & \textbf{FID@10k} & \textbf{FID@50k} \\
\midrule
Baseline (AdamW) & 6.58 & 4.01 \\
+ Window ($\alpha=0.2$,  & \textbf{5.82} & \textbf{3.74} \\
Norm, Proj) \\
\bottomrule
\end{tabular}
\end{table}

\begin{figure}[h!]
    \centering
    \includegraphics[width=0.9\linewidth]{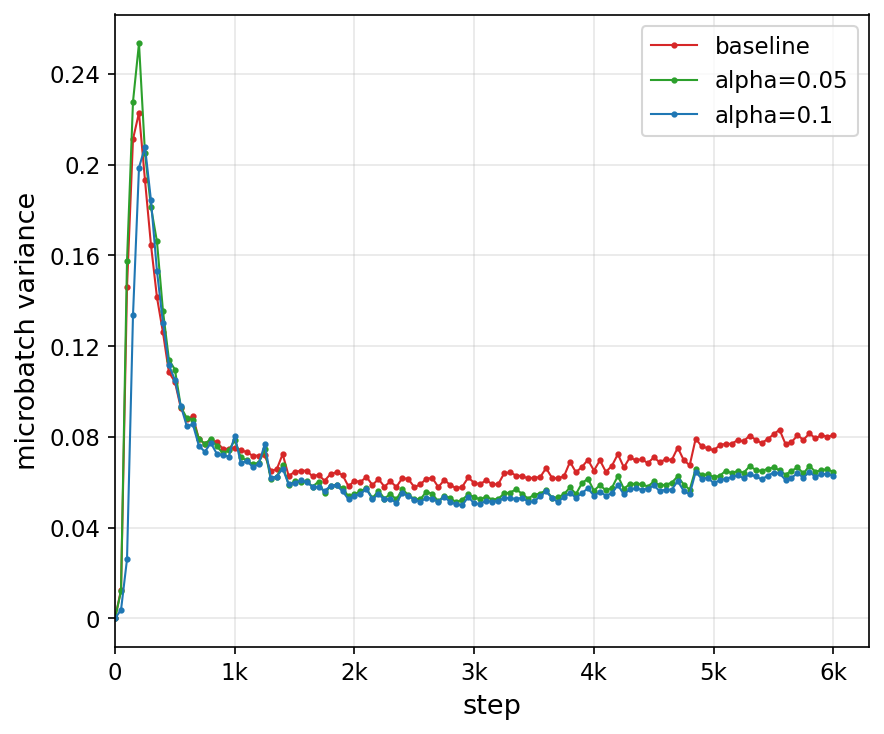}
    \caption{{\bf Microbatch gradient variance with window smoothing.} Total microbatch variance $V_{\mathrm{mb}}$ (Section \ref{sec:stability}) during nanochat depth 24 pretraining, comparing the baseline (red) against window smoothing with $\alpha = 0.05$ (green) and $\alpha = 0.1$ (blue). After the initial warmup, both smoothing runs maintain consistently lower microbatch gradient variance than baseline, with the gap widening later in training.}
    \label{fig:microbatch_var_curves}
\end{figure}
\subsection{Effect of Gradient Smoothing on Representations}

Smoothing parameter updates suggests that at convergence, models may tend toward more coordinated transformations across depth. To quantify this effect on learned representations, we measure two geometric properties of token trajectories: (i) cosine similarity between successive residual updates, and (ii) line shape score (LSS), which captures how close the trajectory lies to a one-dimensional affine space.

{\bf Cosine similarity of residual.} To quantify the change in representation induced by a layer, we consider the differences in token representation across depth $x_\ell$, for $1 \leq \ell \leq L$. For $\ell\in \{1, \ldots, L-1 \}$ the difference $d_\ell = x_{\ell+1} - x_\ell$ corresponds to the contribution of layer $\ell+1$. The similarity of contributions between layers is quantified by $c_\ell = \frac{ \langle{d_\ell, d_{\ell+1} \rangle}} { \|d_\ell\| \|d_{\ell+1}\| }$ (Figure \ref{fig:cos_vs_depth}). 

{\bf Line shape score.} We measure the linearity of token trajectories via the line shape score (LSS) \cite{gai2021mathematicalprincipledeeplearning}, measuring the closeness of a trajectory to a line. Given representations $x_0, \ldots, x_L$, the normalized trajectory $\tilde{x}_0 = x_0$ with 
$$\tilde{x}_\ell = \tilde{x}_{\ell-1} + \frac{x_\ell - x_{\ell-1}}{\|x_\ell - x_{\ell-1}\|_2}$$
The line shape score is then given by $\text{LSS} = \frac{L}{\|\tilde{x}_L - \tilde{x}_0\|_2}$. LSS $\geq 1$, with equality if and only if the trajectory is perfectly co-linear (Figure \ref{fig:lss_vs_alpha}).

We measure the effect of smoothing on representation structure via the Line Shape Score (LSS) and cosine similarity of CLS token trajectories for ViT-B trained with window smoothing for $\alpha \in \{0.1, \ldots, 0.4\}$ relative to the non-smoothed baseline ($\alpha = 0.0$). As the smoothing strength increases, layer-wise contributions to the CLS token become more aligned (Figure \ref{fig:cos_vs_alpha}), and the trend persists when measuring mean cosine similarity of the contribution at individual layers (Figure \ref{fig:cos_vs_depth}). Smoothing also increases the linearity of the CLS token trajectory across layers (Figure \ref{fig:lss_vs_alpha}). Together, these results suggest that Gradient Smoothing may implicitly regularize representation paths, encouraging more coherent evolution of representations through the model depth.

\subsection{Effect of Smoothing on Training Stability}
\label{sec:stability}

\paragraph{Gradient variance.}
To assess downstream improvements on the optimization trajectory, we measure the total stochastic gradient variance within the microbatch. We quantify
\[
V_{\mathrm{mb}}^{(t)} \;=\; \frac{1}{N}\sum_{n=1}^{N}\big\|g^{(t)}_n - \bar g^{(t)}\big\|_2^2,
\]
where at step $t$, $g_n^{(t)}$ is the per-sample parameter gradient and $\bar g^{(t)}$ is the sample mean over the $N$ batch elements at that step. Window smoothing with $\alpha \in \{ 0.05, 0.1 \}$ maintains consistently lower variance than the baseline throughout nanochat pretraining (Figure~\ref{fig:microbatch_var_curves}), suggestive of improved training stability. Full measurement details are deferred to Appendix~\ref{app:variance_empirical}, and a variance contraction bound is shown across layer depth (Proposition \ref{prop:variance_contraction}).



\section{Alignment of Representation Residuals under Gradient Smoothing}
\label{sec:rep_alignment_local_avg}

This section formalizes a consistent empirical observation: \emph{Gradient Smoothing increases the alignment of
representation differences across depth}. We study this through the residual increments
\(
r_\ell := h_{\ell+1}-h_\ell,
\)
and the cosine similarity $\cos(r_{\ell+1},r_\ell)$, which measures how much consecutive residual increments point in the same direction. Our analysis gives lower bounds on the \emph{average cosine} after applying either the baseline optimizer update or its smoothed version. The key mechanism is that Window Smoothing contracts depth-to-depth variation in the update directions.

\paragraph{Residual increments and alignment metric.}
For any depth-stacked parameters $\vartheta=(\vartheta_1,\dots,\vartheta_L)$, define for $\ell=1,\dots,L$, 
\begin{equation}
\label{eq:r_def_main}
r_\ell(\vartheta):=h_{\ell+1}(\vartheta)-h_\ell(\vartheta)
=F\!\big(h_\ell(\vartheta);\vartheta_\ell\big)\in\mathbb{R}^d.
\end{equation}
We measure directional alignment of consecutive residual increments by
\begin{equation}
\label{eq:avg_cos_def_main}
\overline{\cos}(\vartheta)
:=\frac{1}{L-1}\sum_{\ell=1}^{L-1}\cos\!\big(r_{\ell+1}(\vartheta),r_\ell(\vartheta)\big).
\end{equation}
This metric is invariant to the magnitudes $\|r_\ell\|$ and isolates the alignment of representation
\emph{directions}. Empirically, we consistently observe that $\overline{\cos}$ increases under smoothing, as shown in
Figure~\ref{fig:cos_vs_alpha}.

\paragraph{Baseline and smoothed one-step updates.}
Let $u(\theta)=(u_1(\theta),\dots,u_L(\theta))$ denote the blockwise update direction produced by a base optimizer.
From current parameters $\theta$, define the baseline and smoothed one-step updates
\begin{equation}
\label{eq:one_step_updates_main}
\theta^b=\theta-\eta\,u(\theta),
\qquad
\theta^s=\theta-\eta\,S u(\theta),
\end{equation}
where $S$ is a depth-smoothing operator. We focus on the local-averaging smoother (Window Smoothing) with $\alpha\in(0,\tfrac12]$:
\begin{equation}
\label{eq:S_local_avg_main}
(Su)_\ell
=
\begin{cases}
\big(1-\tfrac{\alpha}{2}\big)u_1+\tfrac{\alpha}{2}u_2, & \ell=1,\\[2pt]
(1-\alpha)u_\ell+\tfrac{\alpha}{2}(u_{\ell-1}+u_{\ell+1}), & 2\le \ell\le L-1,\\[2pt]
\big(1-\tfrac{\alpha}{2}\big)u_L+\tfrac{\alpha}{2}u_{L-1}, & \ell=L.
\end{cases}
\end{equation}

\paragraph{Assumptions.}
We assume standard Lipschitz regularity of the block map $F$ and a mild non-degeneracy condition requiring the residual increments to have norms bounded away from zero along the relevant layer-wise trajectories. The formal assumptions are
given in Appendix~\ref{app:rep_alignment_local_avg}.

\subsection{Main Result}
\label{subsec:rep_alignment_main_results}

We now state the main alignment comparison between the baseline and smoothing. Unlike a simplified depth-flat analysis where the parameters are assumed to be equal across depth, this result allows the current
parameters to already have nontrivial depth-to-depth variation. The comparison therefore separates two effects:
the pre-existing depth variation $D\theta$ and the update-induced variation $Du(\theta)$, where $D$ is the first-difference operator $(Dv)_\ell=v_{\ell+1}-v_\ell$.

\begin{theorem}[Alignment Comparison under Smoothing]
\label{thm:avg_cos_general_main}
Assume the regularity conditions of Appendix~\ref{app:rep_alignment_local_avg}
(Assumption~\ref{ass:lipschitz_nondeg}). Let $S$ be the local-averaging smoother \eqref{eq:S_local_avg_main}
with $\alpha\in(0,\tfrac12]$, and let $D$ be the first-difference operator $(Dv)_\ell=v_{\ell+1}-v_\ell$.
Define
\[
\delta := D\theta,
\qquad
w := Du(\theta),
\]
where norms and inner products are taken over the corresponding depth-stacked block vectors. Let
$S^{(1)}$ denote the induced averaging operator on first differences, so that
\[
D(Su)=S^{(1)}Du
\]
as shown in Appendix~\ref{app:rep_alignment_local_avg}, Lemma~\ref{lem:DS_equals_S1D}.

Then the smoothed and baseline one-step updates satisfy
\begin{align}
\label{eq:avg_cos_general_s_main}
\overline{\cos}(\theta^s)
&\ge
1-\frac{4L_h^2M^2}{m^2}
-\frac{4L_\theta^2}{m^2}\cdot\frac{1}{L-1}\,
\big\|\delta-\eta S^{(1)}w\big\|^2,\\
\label{eq:avg_cos_general_b_main}
\overline{\cos}(\theta^b)
&\ge
1-\frac{4L_h^2M^2}{m^2}
-\frac{4L_\theta^2}{m^2}\cdot\frac{1}{L-1}\,
\big\|\delta-\eta w\big\|^2.
\end{align}
Moreover, the difference between these two lower bounds is exactly
\begin{equation}
\label{eq:cert_gap_exact_main}
\frac{4L_\theta^2}{m^2}\cdot\frac{1}{L-1}
\left[
\eta^2\big(\|w\|^2-\|S^{(1)}w\|^2\big)
-2\eta\big\langle \delta,(I-S^{(1)})w\big\rangle
\right].
\end{equation}
Consequently, if $\eta>0$ and $L_\theta>0$, smoothing gives a strictly improved lower bound whenever
\begin{equation}
\label{eq:improvement_condition_main}
\big\langle \delta,(I-S^{(1)})w\big\rangle
<
\frac{\eta}{2}\Big(\|w\|^2-\|S^{(1)}w\|^2\Big).
\end{equation}
In particular, since $S^{(1)}$ is a contraction with
\begin{equation}
\label{eq:mu_star_main}
\|S^{(1)}\|_{\mathrm{op}}
=
\mu_\star
:=
1-\alpha\Big(1-\cos\Big(\frac{\pi}{L}\Big)\Big)
<1,
\end{equation}
the update-induced depth variation satisfies
\begin{equation}
\label{eq:update_roughness_contraction_main}
\|S^{(1)}w\|^2
\le
\mu_\star^2\|w\|^2.
\end{equation}
\end{theorem}

The proof is given in Appendix~\ref{app:rep_alignment_local_avg}. 

The theorem gives a comparison between the baseline and smoothed one-step updates. Both bounds have the same form, but the baseline depends on the raw update difference $w=Du(\theta)$, while the smoothed bound depends on the averaged update difference $S^{(1)}w$. Thus, smoothing affects the bound only through the depth-wise variation of the update.

The exact gap separates two effects. The term
\[
\|w\|^2-\|S^{(1)}w\|^2
\]
measures the reduction in depth-to-depth update variation caused by smoothing, while
\[
\langle \delta,(I-S^{(1)})w\rangle
\]
measures how the removed component of the update variation aligns with the existing parameter variation $\delta=D\theta$. Hence smoothing improves the lower bound whenever the removed component is not too positively aligned with the current depth-wise parameter variation. In the depth-flat case $\delta=0$, this obstruction vanishes, and the improvement follows directly from the contraction of update differences by $S^{(1)}$.

\begin{figure}[h!]
    \centering
    \includegraphics[width=1.0\linewidth]{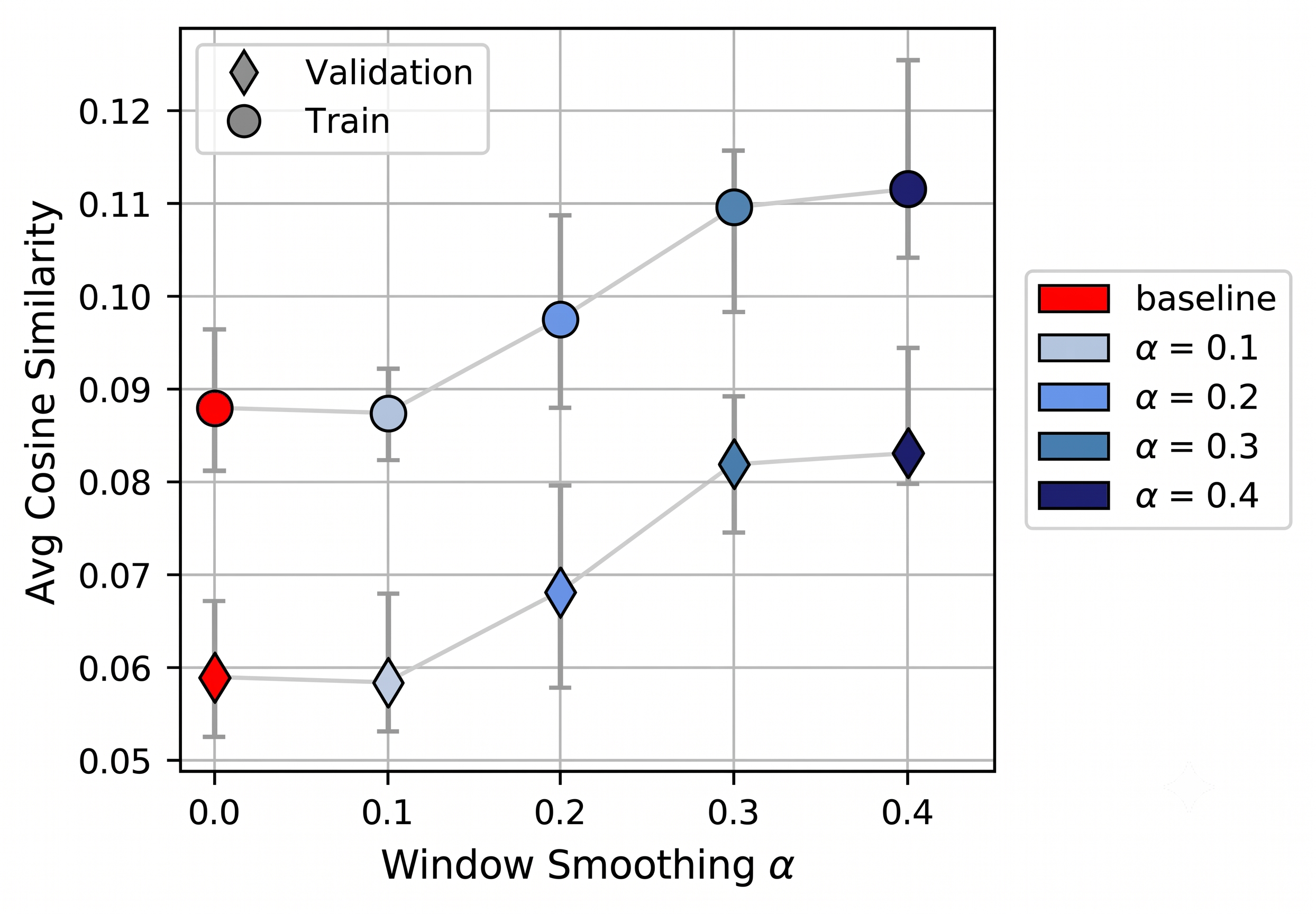}
    \caption{{\bf Layer contributions similarity with increased smoothing.} Average cosine similarity of layer differences $d_\ell = x_{\ell+1} - x_\ell$ for ViT-B CLS token trained on CIFAR-100 (1700 epochs). Means are taken across all layer pairs $(d_i, d_j)$ for $1 \leq i , j \leq L-1$, $i \neq j$. Each point shows the median across 100 images from the validation set (rhombus) and training set (circle); error bars indicate interquartile range (Q1–Q3). As the window smoothing parameter $\alpha \in \{0.1, \ldots, 0.4\}$ increases from the baseline ($\alpha = 0$), mean cosine similarity increases monotonically, suggesting that gradient smoothing encourages more consistent directional updates across layers.}
    \label{fig:cos_vs_alpha}
\end{figure}

\section{Discussion}

{\bf Implicit Biases.} A primary motivation arises from the observed similarity of deep layer representations in transformers and LLMs. If representations and layer contributions exhibit alignment properties at convergence, Gradient Smoothing may be a helpful tool for understanding the implicit structure induced by preconditioning and a mechanism for injecting helpful inductive biases earlier in training. There is an optimal balance when introducing similar regularizers or inductive biases: while we find that moderate layer-wise gradient smoothing is often beneficial, excessive smoothing may reduce important layer-specific information in the updates.

{\bf General Operators.} Gradient Smoothing represents only one family of the broader paradigm of Depth-wise Gradient Augmentation. While we focus on local window averaging, many other coupling operators are possible, including exponential smoothing as well as adaptive or learned schemes based on gradient or representation structure. The strength of coupling may also vary across depth or throughout training, and could be used to steer other characteristics such as gradient variance (Figure \ref{fig:microbatch_var_curves}). More generally, while this work studies smoothing-based coupling, alternative forms of update interaction, including selective coupling, may offer additional benefits. Exploring this broader design space is a promising potential direction for future work.



\section{Related Work}

\paragraph{Transformer Layer Regularity.}
Recent empirical work shows that trained transformers exhibit strong layer-wise structure in their representations and Jacobians \citep{aubry2025transformerblockcouplingcorrelation, gromov2025unreasonableineffectivenessdeeperlayers, li2024residualalignmentuncoveringmechanisms, patrawala2026llm, kapl2025depthgrownmodelsovercomecurse, lad2025remarkablerobustnessllmsstages, jiang2025tracingrepresentationprogressionanalyzing, wolfram2025layerssimilardepthsgenerate}, with some links of such structure to improved generalization. A closely related line of work on \emph{compression valleys} \citep{skean2024doesrepresentationmatterexploring, skean2025layerlayeruncoveringhidden} shows that mid-layer parameters exhibit lower-rank structure than boundary layers, with this low-rank structure often shared across consecutive mid-layers. Such observations also motivate the exploration of layer pruning methods that highlight minimal performance loss \citep{men2024shortgptlayerslargelanguage, gromov2025unreasonableineffectivenessdeeperlayers, jiang2025on, krause2025treadtokenroutingefficient}. Our work connects to this literature by showing that gradient smoothing amplifies these empirical alignment tendencies, while leading to improved performance.

\paragraph{Adaptive Optimization.} Modern deep learning relies heavily on adaptive optimizers that scale updates based on gradient statistics. Adam \citep{kingma2017adam} and AdamW \citep{loshchilov2019decoupled} maintain adaptive per-parameter estimates of first and second moments, enabling effective training across all modern architectures. Second-order methods such as Shampoo \citep{gupta2018shampoopreconditionedstochastictensor} and SOAP \citep{vyas2025soap} approximate full-matrix preconditioning for improved convergence, while Muon \citep{jordan2024muon} applies momentum in the orthogonalized gradient space. These methods focus on adaptively scaling updates per parameter group, whereas gradient smoothing operates on a complementary axis: \emph{depth-wise augmentations} of updates across optimization steps. Our approach is agnostic to the choice of base optimizer and can be combined with existing adaptive methods to stabilize training by amplifying gradient information across depth.

\paragraph{Pre-conditioned Optimization.} Natural gradient descent \citep{amari1998} and its approximations such as K-FAC \citep{martens2020optimizingneuralnetworkskroneckerfactored, george2021fastapproximatenaturalgradient, eschenhagen2024kroneckerfactoredapproximatecurvaturemodern, lin2024structuredinversefreenaturalgradient, nagwekar2025guideddescentoptimizationalgorithms} precondition updates using curvature information to accelerate convergence. Our method adopts a pre-conditioned perspective but focuses on depth structure rather than per-layer curvature, applying smoothing across transformer depth with minimal computational overhead.

\textbf{Structure in Hidden Representations.} Many works study the emergent structure of deep network representations \citep{wang2024understanding, parker2023neural, zangrando2025provableemergencedeepneural, garrod2024unifying, wang2024progressive, hoyt2021probing, arous2024highdimensional, zarka2021separationconcentrationdeepnetworks, shaul22a-pmlr, papyan2020traces, sukenik, papyan2017nc, Zhou2025, wang2026visualpromptagnosticevolution, fisher2024pushingboundariesmixupsinfluence}, especially in transformers. In LLMs, recent works identify uniform token structures \citep{wu2024linguisticcollapse, shai2025transformersrepresentbeliefstate, piotrowski2025constrainedbeliefupdatesexplain, skean2025layerlayeruncoveringhidden, skean2024doesrepresentationmatterexploring} and low-dimensional hidden trajectories \citep{song2025bridgingdimensionalchasmuncover, sarfati2024linesthoughtlargelanguage}. Our work proposes an approach to regularize representation trajectories across depth by augmenting the gradient updates.

\section{Conclusion}
We introduced \emph{Depth-wise Gradient Augmentation}, a general optimization paradigm in which the collection of block-wise optimizer updates is transformed along the depth dimension. Within this framework, we studied \emph{Gradient Smoothing}, a family of depth-wise smoothing methods, instantiated through local Window Smoothing. The resulting method is compatible with arbitrary base optimizers and consistently improves training across a diverse range of settings, including LLM pretraining, RL post-training for reasoning, image classification, and diffusion modeling. Beyond these performance gains, we showed that Gradient Smoothing promotes more structured representation evolution across depth, leading to greater alignment and linearity of representation trajectories. We further provided a theoretical characterization of this phenomenon by interpreting Gradient Smoothing as a structured depth-wise preconditioning method that contracts variation in block-wise updates, yielding improved bounds on the alignment of representation differences.

More broadly, our results suggest that explicitly exploiting cross-depth structure during optimization can improve both optimization and generalization in modern deep networks. We hope that the proposed Depth-wise Gradient Augmentation framework motivates future work on richer depth-wise update transformations, adaptive smoothing operators, learned augmentation schemes, and other forms of structured optimization that leverage the organization of repeated-block architectures.




\section*{Impact Statement}

This paper presents work whose goal is to advance the field of Machine Learning. There are some potential societal consequences of improved optimization methods, including downstream impacts on systems used in science and decision-making. At the same time, the techniques introduced in this work are general-purpose and do not target any specific application domain. We do not foresee immediate negative societal impacts arising directly from this work.



\newpage

\bibliography{paper}
\bibliographystyle{icml2026}

\newpage
\appendix
\onecolumn
%


\section{Appendix}

\subsection{Specialization of Smoothing to Adam, AdamW, and Muon}
\label{sec:adam_specialization}

In our experiments, the base optimizer $\mathcal{U}^{(t)}$ is typically Adam \citep{kingma2017adam}, AdamW \citep{loshchilov2019decoupled}, or Muon~\citep{jordan2024muon}.
For completeness, we briefly recall the update rule and describe how gradient smoothing composes with it.

\paragraph{Adam updates.}
Let
\(
g^{(t)}=\nabla_\theta \mathcal{L}(\theta^{(t)},\phi^{(t)})
\)
denote the gradient with respect to the repeated-block parameters. Note that the remaining parameters $\phi$ are updated in the same manner as defined below, but are just not involved in the smoothing process.
Adam maintains exponential moving averages of first and second moments,
\begin{align}
m^{(t)} &= \beta_1 m^{(t-1)} + (1-\beta_1) g^{(t)}, \\
v^{(t)} &= \beta_2 v^{(t-1)} + (1-\beta_2)\big(g^{(t)}\odot g^{(t)}\big),
\end{align}
with bias-corrected versions $\hat m^{(t)}$ and $\hat v^{(t)}$.
The update direction is
\begin{equation}
u^{(t)} := \frac{\hat m^{(t)}}{\sqrt{\hat v^{(t)}} + \varepsilon},
\end{equation}
and the block parameters are updated as
\begin{equation}
\theta^{(t+1)} = \theta^{(t)} - \eta\, u^{(t)}.
\end{equation}

Equivalently, Adam applies a diagonal, coordinate-wise preconditioner
\[
u^{(t)} = A^{(t)} g^{(t)}, 
\qquad
A^{(t)} := \mathrm{diag}\!\left(\frac{1}{\sqrt{\hat v^{(t)}}+\varepsilon}\right).
\]

\paragraph{AdamW.}
AdamW applies decoupled weight decay, yielding the block-parameter update
\begin{equation}
\theta^{(t+1)} = \theta^{(t)} - \eta\big(u^{(t)} + \lambda\, \theta^{(t)}\big),
\end{equation}
with non-block parameters $\phi$ updated by the base optimizer in the standard way.

\paragraph{Adam/AdamW with Gradient Smoothing.} The smoothed block update is then
\[
\tilde u^{(t)} = P\, u^{(t)}.
\]
Gradient smoothing yields
\begin{align}
\theta^{(t+1)} &= \theta^{(t)} - \eta\, \tilde u^{(t)} \quad \text{(Adam)}, \\
\theta^{(t+1)} &= \theta^{(t)} - \eta\big(\tilde u^{(t)} + \lambda\, \theta^{(t)}\big) \quad \text{(AdamW)}.
\end{align}

Equivalently, the effective update can be written as
\[
\theta^{(t+1)} = \theta^{(t)} - \eta\, (P A^{(t)}) g^{(t)},
\]
highlighting that gradient smoothing composes a depth-structured preconditioner with Adam’s coordinate-wise preconditioning, with $\phi$ being updated as in standard Adam/AdamW.

\paragraph{Muon.}

Muon \citep{jordan2024muon} treats each repeated-block parameter as a matrix and produces an approximately orthogonal update direction. It maintains a momentum buffer
\[
M^{(t)} = \mu\, M^{(t-1)} + g^{(t)},
\]
and forms the update direction by orthogonalizing $M^{(t)}$ via a Newton--Schulz iteration,
\[
u^{(t)} := \mathrm{NS}\!\big(M^{(t)}\big) \;\approx\; M^{(t)}\big((M^{(t)})^\top M^{(t)}\big)^{-1/2}.
\]
The block parameters are updated as $\theta^{(t+1)} = \theta^{(t)} - \eta\,u^{(t)}$. In contrast to Adam's diagonal preconditioning, $u^{(t)}$ has singular values approximately equal to one, so its per-step magnitude is largely decoupled from that of the raw gradient.

\paragraph{Muon with Gradient Smoothing.}
The smoothed block update composes with Muon in the same way as with Adam:
\[
\tilde u^{(t)} = P\, u^{(t)},
\qquad
\theta^{(t+1)} = \theta^{(t)} - \eta\, \tilde u^{(t)}.
\]
For matrix-valued block parameters, $P$ is applied to the orthogonalized update produced by the Newton--Schulz iteration; for non-matrix block parameters (e.g., normalization scales), to which Newton--Schulz does not apply, $P$ is applied directly to the base-optimizer update for that parameter.

\subsection{Alignment of Representation Residuals under Gradient Smoothing}
\label{app:rep_alignment_local_avg}

\paragraph{Residual increments and alignment metric.}
For any depth-stacked parameter vector $\vartheta=(\vartheta_1,\dots,\vartheta_L)$, define the residual increments
\begin{equation}
\label{eq:r_def}
r_\ell(\vartheta) \;:=\; h_{\ell+1}(\vartheta)-h_\ell(\vartheta)
\;=\; F\!\big(h_\ell(\vartheta);\vartheta_\ell\big)\in\mathbb{R}^d,
\qquad \ell=1,\dots,L,
\end{equation}
and the normalized directions
\begin{equation}
\label{eq:d_def}
d_\ell(\vartheta)\;:=\;\frac{r_\ell(\vartheta)}{\|r_\ell(\vartheta)\|}\in\mathbb{R}^d
\quad\text{(defined whenever $r_\ell(\vartheta)\neq 0$).}
\end{equation}
Our empirical metric is
\begin{equation}
\label{eq:cos_metric}
\cos\!\big(r_{\ell+1}(\vartheta),r_\ell(\vartheta)\big)
\;=\;\big\langle d_{\ell+1}(\vartheta),d_\ell(\vartheta)\big\rangle,
\qquad \ell=1,\dots,L-1.
\end{equation}
Working with $d_\ell$ isolates \emph{directional} alignment from magnitude effects (e.g.\ $\|r_{\ell+1}-r_\ell\|$ may be large even when $\cos(r_{\ell+1},r_\ell)$ is close to $1$ due to norm mismatch).

\paragraph{Assumptions.}
We work on a region containing the trajectories $\{h_\ell(\theta^b)\}$ and $\{h_\ell(\theta^s)\}$ and impose the following standard regularity assumptions.

\begin{assumption}
\label{ass:lipschitz_nondeg}
There exist constants $L_h,L_\theta\ge 0$ such that for all $h,\bar h\in\mathbb{R}^d$ and $\theta,\bar\theta\in\mathbb{R}^p$ in the region of interest,
\begin{equation}
\label{eq:F_Lip}
\|F(h;\theta)-F(\bar h;\bar\theta)\|
\;\le\;L_h\|h-\bar h\|+L_\theta\|\theta-\bar\theta\|.
\end{equation}

Moreover, there exist constants $0<m\le M<\infty$ such that, for each
$\vartheta\in\{\theta^b,\theta^s\}$,
\[
m \le \min_{\ell\in[L]}\|r_\ell(\vartheta)\|,
\qquad
\max_{\ell\in[L]}\|r_\ell(\vartheta)\|\le M .
\]
\end{assumption}

\begin{lemma}
\label{lem:cos_unit_dist}
For any nonzero $a,b\in\mathbb{R}^d$,
\begin{equation}
\label{eq:cos_identity}
1-\cos(a,b)\;=\;\frac{1}{2}\left\|\frac{a}{\|a\|}-\frac{b}{\|b\|}\right\|^2.
\end{equation}
\end{lemma}
\begin{proof}
Expand $\|a/\|a\|-b/\|b\|\|^2=2-2\langle a/\|a\|,b/\|b\|\rangle$ and rearrange.
\end{proof}

\begin{lemma}
\label{lem:norm_lipschitz}
For any nonzero $a,b\in\mathbb{R}^d$,
\begin{equation}
\label{eq:norm_lipschitz}
\left\|\frac{a}{\|a\|}-\frac{b}{\|b\|}\right\|
\;\le\;\frac{2\|a-b\|}{\min(\|a\|,\|b\|)}.
\end{equation}
\end{lemma}
\begin{proof}
Write
\[
\frac{a}{\|a\|}-\frac{b}{\|b\|}
=\frac{a-b}{\|a\|}+b\Big(\frac{1}{\|a\|}-\frac{1}{\|b\|}\Big),
\]
so by the triangle inequality,
\[
\left\|\frac{a}{\|a\|}-\frac{b}{\|b\|}\right\|
\le \frac{\|a-b\|}{\|a\|}+\|b\|\cdot\frac{\big|\|b\|-\|a\|\big|}{\|a\|\|b\|}
\le \frac{\|a-b\|}{\|a\|}+\frac{\|a-b\|}{\|a\|}
=\frac{2\|a-b\|}{\|a\|},
\]
and symmetrizing in $(a,b)$ yields \eqref{eq:norm_lipschitz}.
\end{proof}

\paragraph{A bound on cosine misalignment.}
\begin{lemma}
\label{lem:turning_bound}
Under Assumption~\ref{ass:lipschitz_nondeg}, for any $\vartheta\in\{\theta^b,\theta^s\}$ and any $\ell\in[L-1]$,
\begin{equation}
\label{eq:turning_bound_pointwise}
1-\cos\!\big(r_{\ell+1}(\vartheta),r_\ell(\vartheta)\big)
\;\le\;\frac{4}{m^2}\Big(L_h^2\|r_\ell(\vartheta)\|^2 + L_\theta^2\|\vartheta_{\ell+1}-\vartheta_\ell\|^2\Big).
\end{equation}
Consequently,
\begin{equation}
\label{eq:avg_turning_bound}
\frac{1}{L-1}\sum_{\ell=1}^{L-1}\cos\!\big(r_{\ell+1}(\vartheta),r_\ell(\vartheta)\big)
\;\ge\;
1-\frac{4L_h^2M^2}{m^2}
-\frac{4L_\theta^2}{m^2}\cdot\frac{1}{L-1}\sum_{\ell=1}^{L-1}\|\vartheta_{\ell+1}-\vartheta_\ell\|^2.
\end{equation}
\end{lemma}
\begin{proof}
Fix $\vartheta$ and $\ell$. By Lemma~\ref{lem:cos_unit_dist} and Lemma~\ref{lem:norm_lipschitz},
\[
1-\cos(r_{\ell+1},r_\ell)
=\frac12\|d_{\ell+1}-d_\ell\|^2
\le \frac12\left(\frac{2\|r_{\ell+1}-r_\ell\|}{\min(\|r_{\ell+1}\|,\|r_\ell\|)}\right)^2
\le \frac{2}{m^2}\|r_{\ell+1}-r_\ell\|^2.
\]
Using \eqref{eq:r_def} and the Lipschitz property \eqref{eq:F_Lip},
\[
\|r_{\ell+1}-r_\ell\|
=\big\|F(h_{\ell+1};\vartheta_{\ell+1})-F(h_\ell;\vartheta_\ell)\big\|
\le L_h\|h_{\ell+1}-h_\ell\|+L_\theta\|\vartheta_{\ell+1}-\vartheta_\ell\|
= L_h\|r_\ell\|+L_\theta\|\vartheta_{\ell+1}-\vartheta_\ell\|.
\]
Squaring and applying $(a+b)^2\le 2a^2+2b^2$ yields
\[
\|r_{\ell+1}-r_\ell\|^2
\le 2L_h^2\|r_\ell\|^2+2L_\theta^2\|\vartheta_{\ell+1}-\vartheta_\ell\|^2.
\]
Combining the last two inequalities gives \eqref{eq:turning_bound_pointwise}. Averaging \eqref{eq:turning_bound_pointwise} over $\ell$ and using $\|r_\ell(\vartheta)\|\le M$ gives \eqref{eq:avg_turning_bound}.
\end{proof}

\paragraph{Local averaging smoother and its action on update differences.}
We now take $S$ to be the \emph{local averaging} (tridiagonal) smoother with strength $\alpha\in(0,\tfrac12]$:
\begin{equation}
\label{eq:S_local_avg}
(Su)_\ell
=
\begin{cases}
\big(1-\tfrac{\alpha}{2}\big)u_1+\tfrac{\alpha}{2}u_2, & \ell=1,\\[2pt]
(1-\alpha)u_\ell+\tfrac{\alpha}{2}(u_{\ell-1}+u_{\ell+1}), & 2\le \ell\le L-1,\\[2pt]
\big(1-\tfrac{\alpha}{2}\big)u_L+\tfrac{\alpha}{2}u_{L-1}, & \ell=L.
\end{cases}
\end{equation}
Let $D\in\mathbb{R}^{(L-1)\times L}$ denote the first-difference operator $(Dv)_\ell=v_{\ell+1}-v_\ell$.

\begin{lemma}[Induced local averaging on first differences]
\label{lem:DS_equals_S1D}
Let $D$ be the first-difference operator $(Du)_\ell=u_{\ell+1}-u_\ell$ for $\ell=1,\dots,L-1$.
Let $S$ be the local averaging operator defined in \eqref{eq:S_local_avg}.
Define $S^{(1)}\in\mathbb{R}^{(L-1)\times(L-1)}$ by
\[
(S^{(1)}w)_\ell=
\begin{cases}
(1-\alpha)w_1+\tfrac{\alpha}{2}w_2, & \ell=1,\\[2pt]
(1-\alpha)w_\ell+\tfrac{\alpha}{2}(w_{\ell-1}+w_{\ell+1}), & 2\le \ell\le L-2,\\[2pt]
(1-\alpha)w_{L-1}+\tfrac{\alpha}{2}w_{L-2}, & \ell=L-1.
\end{cases}
\]
Then for all block-vectors $u$,
\[
D(Su)=S^{(1)}(Du).
\]
\end{lemma}

\begin{proof}
Let $w:=Du$, i.e.\ $w_\ell=u_{\ell+1}-u_\ell$.
For interior $\ell\in\{2,\dots,L-2\}$, expanding $(Su)_{\ell+1}-(Su)_\ell$ using \eqref{eq:S_local_avg} gives
\[
(D(Su))_\ell=(1-\alpha)w_\ell+\tfrac{\alpha}{2}w_{\ell-1}+\tfrac{\alpha}{2}w_{\ell+1}.
\]
At $\ell=1$,
\[
(D(Su))_1=(Su)_2-(Su)_1=(1-\alpha)w_1+\tfrac{\alpha}{2}w_2.
\]
At $\ell=L-1$,
\[
(D(Su))_{L-1}=(Su)_L-(Su)_{L-1}=(1-\alpha)w_{L-1}+\tfrac{\alpha}{2}w_{L-2}.
\]
These are exactly the defining rules of $S^{(1)}$ applied to $w$, hence $D(Su)=S^{(1)}w=S^{(1)}(Du)$.
\end{proof}

\begin{lemma}[Eigenvalues of induced averaging on first differences]
\label{lem:Sdiff_eigs}
Let $n\ge 1$ and define $\widetilde S^{(n)}\in\mathbb{R}^{n\times n}$ by the tridiagonal Toeplitz form
\[
(\widetilde S^{(n)}w)_i
=
(1-\alpha)w_i+\frac{\alpha}{2}\big(\mathbf{1}_{\{i>1\}}w_{i-1}+\mathbf{1}_{\{i<n\}}w_{i+1}\big),
\qquad i=1,\dots,n,
\]
with $\alpha\in(0,1)$. Then $\widetilde S^{(n)}$ is symmetric and has eigenpairs
\[
v^{(k)}_j=\sin\!\Big(\frac{\pi k j}{n+1}\Big),
\qquad
\widetilde\mu_k^{(n)} = 1-\alpha\Big(1-\cos\Big(\frac{\pi k}{n+1}\Big)\Big),
\qquad k=1,\dots,n.
\]
In particular, $\|\widetilde S^{(n)}\|_{\mathrm{op}}=\widetilde\mu_1^{(n)}<1$.
\end{lemma}

\begin{proof}
Fix $k\in\{1,\dots,n\}$ and set $t:=\pi k/(n+1)$. For $2\le j\le n-1$,
\[
(\widetilde S^{(n)}v^{(k)})_j
=(1-\alpha)\sin(jt)+\frac{\alpha}{2}\big(\sin((j-1)t)+\sin((j+1)t)\big)
=\big(1-\alpha+\alpha\cos t\big)\sin(jt),
\]
using $\sin((j-1)t)+\sin((j+1)t)=2\sin(jt)\cos t$.
At $j=1$ and $j=n$ the same identity holds since the missing neighbor corresponds to $\sin(0)=\sin((n+1)t)=0$.
Thus $v^{(k)}$ is an eigenvector with eigenvalue $1-\alpha+\alpha\cos t=1-\alpha(1-\cos t)$.
The operator norm equals the largest absolute eigenvalue $\widetilde\mu_1^{(n)}$ since $\widetilde S^{(n)}$ is symmetric.
\end{proof}
\begin{proposition}
\label{prop:Du_contraction}
Let $S^{(1)}\in\mathbb{R}^{(L-1)\times(L-1)}$ be the induced operator on differences from Lemma~\ref{lem:DS_equals_S1D}
(i.e.\ with boundary diagonal $1-\alpha$). Then for all $w\in\mathbb{R}^{L-1}$,
\begin{equation}
\label{eq:S1_contraction_new}
\|S^{(1)}w\|
\;\le\;\mu_\star\,\|w\|,
\qquad
\mu_\star := 1-\alpha\Big(1-\cos\Big(\frac{\pi}{L}\Big)\Big)\;<\;1,
\end{equation}
and hence
\begin{equation}
\label{eq:energy_drop_new}
\|w\|^2-\|S^{(1)}w\|^2
\;\ge\;\big(1-\mu_\star^2\big)\,\|w\|^2.
\end{equation}
\end{proposition}

\begin{proof}
By Lemma~\ref{lem:Sdiff_eigs} with $n=L-1$, the eigenvalues of $S^{(1)}=\widetilde S^{(L-1)}$ are
$\widetilde\mu_k^{(L-1)} = 1-\alpha(1-\cos(\pi k/L))$ for $k=1,\dots,L-1$.
Thus $\|S^{(1)}\|_{\mathrm{op}}=\widetilde\mu_1^{(L-1)}=\mu_\star<1$, giving \eqref{eq:S1_contraction_new}.
Then \eqref{eq:energy_drop_new} follows from $\|S^{(1)}w\|^2\le \mu_\star^2\|w\|^2$.
\end{proof}

\paragraph{General one-step bound for the average cosine alignment.}
Define the average cosine across consecutive residual increments:
\begin{equation}
\label{eq:avg_cos_def}
\overline{\cos}(\vartheta)
:=
\frac{1}{L-1}\sum_{\ell=1}^{L-1}
\cos\!\big(r_{\ell+1}(\vartheta),r_\ell(\vartheta)\big).
\end{equation}
We now state the general one-step comparison, allowing the current parameters to have pre-existing depth variation.

\begin{corollary}
\label{cor:avg_cos_general}
Assume Assumption~\ref{ass:lipschitz_nondeg}. Let $S$ be the local-averaging smoother \eqref{eq:S_local_avg}
with parameter $\alpha\in(0,\tfrac12]$, and let $S^{(1)}$ be the induced $(L-1)\times(L-1)$ averaging operator
from Lemma~\ref{lem:DS_equals_S1D}. Define
\[
\delta := D\theta,
\qquad
w := Du(\theta),
\]
where $D$ is the first-difference operator and all norms and inner products are taken over depth-stacked block vectors.
Then the smoothed and baseline one-step updates satisfy
\begin{align}
\label{eq:avg_cos_general_s}
\overline{\cos}(\theta^s)
&\ge
1-\frac{4L_h^2M^2}{m^2}
-\frac{4L_\theta^2}{m^2}\cdot\frac{1}{L-1}\,
\big\|\delta-\eta S^{(1)}w\big\|^2,\\
\label{eq:avg_cos_general_b}
\overline{\cos}(\theta^b)
&\ge
1-\frac{4L_h^2M^2}{m^2}
-\frac{4L_\theta^2}{m^2}\cdot\frac{1}{L-1}\,
\big\|\delta-\eta w\big\|^2.
\end{align}
Moreover, the difference between the two lower bounds admits the exact identity
\begin{equation}
\label{eq:cert_gap_exact}
\Big[\text{RHS of \eqref{eq:avg_cos_general_s}}\Big]
-
\Big[\text{RHS of \eqref{eq:avg_cos_general_b}}\Big]
=
\frac{4L_\theta^2}{m^2}\cdot\frac{1}{L-1}
\left[
\eta^2\big(\|w\|^2-\|S^{(1)}w\|^2\big)
-2\eta\big\langle \delta,(I-S^{(1)})w\big\rangle
\right].
\end{equation}
In particular, if \(\eta>0\) and \(L_\theta>0\), smoothing yields a strictly improved lower bound whenever
\begin{equation}
\label{eq:improvement_condition}
\big\langle \delta,(I-S^{(1)})w\big\rangle
<
\frac{\eta}{2}\Big(\|w\|^2-\|S^{(1)}w\|^2\Big).
\end{equation}
A simple sufficient condition for such is
\begin{equation}
\label{eq:improvement_condition_sufficient}
\|\delta\|
<
\frac{\eta}{2}\cdot
\frac{1-\mu_\star^2}{\|I-S^{(1)}\|_{\mathrm{op}}}\,\|w\|,
\qquad
\mu_\star:=\|S^{(1)}\|_{\mathrm{op}}<1.
\end{equation}
\end{corollary}

\begin{proof}
The bounds \eqref{eq:avg_cos_general_s}--\eqref{eq:avg_cos_general_b} follow by applying
Lemma~\ref{lem:turning_bound} with $\vartheta=\theta^s$ and $\vartheta=\theta^b$. Indeed,
\[
D\theta^s
=
D(\theta-\eta Su)
=
D\theta-\eta D(Su)
=
\delta-\eta S^{(1)}w,
\]
where we used Lemma~\ref{lem:DS_equals_S1D}, and similarly
\[
D\theta^b
=
D(\theta-\eta u)
=
\delta-\eta w.
\]
Substituting these identities into \eqref{eq:avg_turning_bound} gives
\eqref{eq:avg_cos_general_s} and \eqref{eq:avg_cos_general_b}.

For \eqref{eq:cert_gap_exact}, subtract the right-hand sides of
\eqref{eq:avg_cos_general_s} and \eqref{eq:avg_cos_general_b}. The only changing term is the squared depth-difference
term, and
\[
\|\delta-\eta S^{(1)}w\|^2-\|\delta-\eta w\|^2
=
\eta^2\big(\|S^{(1)}w\|^2-\|w\|^2\big)
+
2\eta\big\langle \delta,(I-S^{(1)})w\big\rangle.
\]
Multiplying by
\[
-\frac{4L_\theta^2}{m^2}\cdot\frac{1}{L-1}
\]
yields \eqref{eq:cert_gap_exact}. Condition \eqref{eq:improvement_condition} makes the right-hand side strictly
positive.

Finally, \eqref{eq:improvement_condition_sufficient} follows from Cauchy--Schwarz and the contraction of
$S^{(1)}$. Specifically,
\[
\big\langle \delta,(I-S^{(1)})w\big\rangle
\le
\|\delta\|\,\|I-S^{(1)}\|_{\mathrm{op}}\|w\|,
\]
while Proposition~\ref{prop:Du_contraction} gives
\[
\|w\|^2-\|S^{(1)}w\|^2
\ge
(1-\mu_\star^2)\|w\|^2.
\]
Combining these two inequalities gives the stated sufficient condition.
\end{proof}

\begin{remark}[Depth-flat special case]
\label{rem:depth_flat_special_case}
If the current parameters are depth-flat, $\theta_1=\cdots=\theta_L$, then $\delta=D\theta=0$. In this case,
Corollary~\ref{cor:avg_cos_general} reduces to a purely update-induced comparison:
\begin{equation}
\label{eq:avg_cos_lower_smoothed_depth_flat}
\overline{\cos}(\theta^s)
\ge
1-\frac{4L_h^2M^2}{m^2}
-\frac{4L_\theta^2\eta^2}{m^2}\cdot
\frac{\|S^{(1)}w\|^2}{L-1}.
\end{equation}
The corresponding improvement over the baseline lower bound is
\begin{align}
\label{eq:avg_cos_cert_gap_depth_flat}
\Big[\text{RHS of \eqref{eq:avg_cos_lower_smoothed_depth_flat}}\Big]
-
\Big[\text{same bound with $\|S^{(1)}w\|^2$ replaced by $\|w\|^2$}\Big]
&=
\frac{4L_\theta^2\eta^2}{m^2}\cdot
\frac{\|w\|^2-\|S^{(1)}w\|^2}{L-1} \\
&\ge
\frac{4L_\theta^2\eta^2}{m^2}\cdot
\frac{1-\mu_\star^2}{L-1}\,\|w\|^2.
\end{align}
Thus, when $\eta>0$, $L_\theta>0$, and $w\neq 0$, depth-flat parameters yield a strict improvement in the lower bound. This special case isolates the mechanism: local averaging improves the bound solely by
contracting depth-to-depth variation in the optimizer update. The main result, Corollary~\ref{cor:avg_cos_general},
does not require this depth-flat condition and additionally accounts for interactions with pre-existing depth
variation $D\theta$.
\end{remark}

\FloatBarrier
\begin{figure*}[h]
    \centering
    \includegraphics[width=0.75\linewidth]{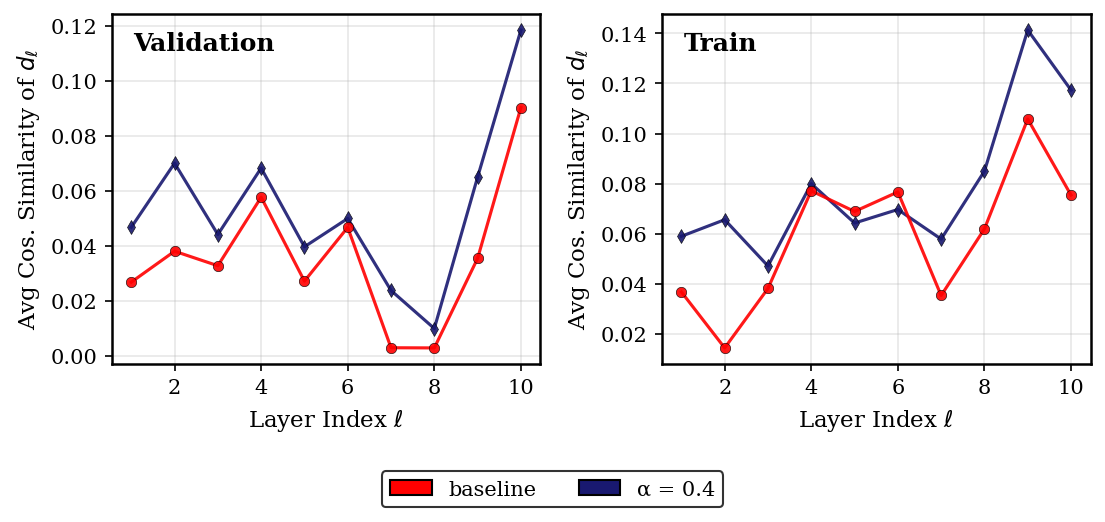}
    \caption{{\bf Gradient smoothing increases layer contribution alignment.} CLS token trajectories of ViT-B trained on CIFAR-100 for 1700 epochs, comparing baseline ($\alpha = 0$, red) against heavy smoothing ($\alpha = 0.4$, navy). For each layer difference $d_\ell = x_{\ell+1} - x_\ell$, we compute the mean cosine similarity with all other layer differences $d_j$ ($j \neq \ell$). Across both validation and training sets, the smoothed model exhibits consistently higher mean cosine similarity at most layers, indicating that gradient smoothing encourages more coherent directional updates throughout the network.}  
        \label{fig:cos_vs_depth}
\end{figure*}

\begin{figure}[h]
    \centering
    \includegraphics[width=0.5\linewidth]{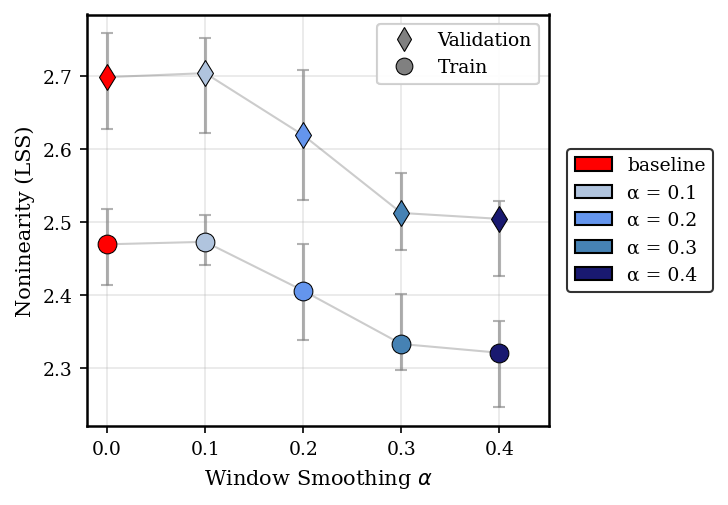}
    \caption{{\bf Greater linearity with increased smoothing.} Line Shape Score (LSS) of CLS token trajectories in ViT-B trained on CIFAR-100 (1700 epochs). Each point shows the median across 100 validation (rhombus) and training (circle) images, while error bars indicate interquartile range (Q1–Q3). As the window smoothing parameter $\alpha \in \{0.1, \ldots, 0.4 \}$ increases from the baseline ($\alpha$ = 0), median LSS decreases monotonically, indicating that gradient smoothing produces more linear CLS token trajectories.}
    \label{fig:lss_vs_alpha}
\end{figure}

\subsection{Variance-Reduction of Gradients with Smoothing}
\label{app:variance_reduction}

Here we provide a simple variance-based analysis of Gradients with the local window operator. Let $U \in \mathbb{R}^{Ld}$ denote the stacked stochastic (e.g.\ from minibatches) block update, and define the smoothed update by
\[
\tilde U = (S \otimes I_d) U,
\]
where $S$ is the depth-smoothing matrix.

\begin{proposition}[Variance under a single step of window smoothing]
\label{prop:variance_contraction}
Let $U \in \mathbb{R}^{Ld}$ be a random vector with $\mathbb{E}\|U\|_2^2 < \infty$ and covariance $\Sigma = \mathrm{Cov}(U)$, and let $\tilde U = (S \otimes I_d) U$ with $S$ the local window operator \eqref{eq:S_local_avg}. For any $\alpha \in [0,1]$,
\[
\mathbb{E}\|\tilde U - \mathbb{E}\tilde U\|_2^2 \;\le\; \mathbb{E}\|U - \mathbb{E}U\|_2^2.
\]
\end{proposition}

\begin{proof}
Since $S$ is symmetric with eigenvalues
\[
\lambda_k(S) = 1 - \alpha\!\left(1 - \cos\!\frac{k\pi}{L}\right), \qquad k = 0, \dots, L-1,
\]
we have $\|S\|_2 = 1$ for $\alpha \in [0,1]$.

By linearity, $\mathrm{Cov}(\tilde U) = (S \otimes I_d)\,\Sigma\,(S \otimes I_d)^\top$, so the total stochastic update variance is
\[
\mathbb{E}\|\tilde U - \mathbb{E}\tilde U\|_2^2
= \mathrm{tr}\!\bigl((S \otimes I_d)\,\Sigma\,(S \otimes I_d)^\top\bigr)
= \mathrm{tr}\!\bigl(((S^\top S) \otimes I_d)\,\Sigma\bigr).
\]
Let $A := (S^\top S) \otimes I_d$. Since $\Sigma$ is a covariance matrix, $\Sigma \succeq 0$, and clearly $A \succeq 0$ as well. Then $A \preceq \|A\|_2\,I$, which implies
\[
\mathrm{tr}(A\Sigma) = \mathrm{tr}\!\bigl(\Sigma^{1/2} A\,\Sigma^{1/2}\bigr) \le \|A\|_2\,\mathrm{tr}(\Sigma).
\]
Since
\[
\|A\|_2 = \|(S^\top S) \otimes I_d\|_2 = \|S^\top S\|_2 = \|S\|_2^2 = 1,
\]
we obtain
\[
\mathbb{E}\|\tilde U - \mathbb{E}\tilde U\|_2^2 \le \mathrm{tr}(\Sigma) = \mathbb{E}\|U - \mathbb{E}U\|_2^2. \qedhere
\]
\end{proof}

Thus, for $\alpha \in [0,1]$, window smoothing update variance is at most that of the blockwise stochastic update.

\subsection{Empirical Measurement of Gradient Variance}
\label{app:variance_empirical}

To complement the variance contraction bound of Proposition~\ref{prop:variance_contraction}, we directly measure two notions of stochastic gradient variance during pretraining and compare them across baseline and smoothing runs.

\paragraph{Setup.}
We train three otherwise-identical \texttt{nanochat} d24 models (1.38B parameters, 24 transformer blocks) for 7,308 optimization steps: a baseline run and two window-smoothing runs with $\alpha \in \{0.05,\,0.1\}$. Every $50$ steps, before applying the optimizer update, we record per-microbatch gradients for each parameter group $(r,\ell)$, where $r$ indexes one of $R=6$ roles within a block (the four attention projections and the two MLP linears) and $\ell \in \{1,\ldots,L\}$ indexes block depth. Let $g_{r,\ell,n}$ denote the flattened gradient for role $r$, layer $\ell$, and microbatch $n$, with $N$ microbatches per logging step.

\paragraph{Microbatch variance.}
The microbatch variance measures the across-sample noise of the stochastic gradient at each parameter group:
\[
V^{\mathrm{mb}}_{r,\ell}
\;:=\;
\frac{1}{N}\sum_{n=1}^{N}\bigl\|g_{r,\ell,n}-\bar g_{r,\ell}\bigr\|_2^2,
\qquad
\bar g_{r,\ell}\;:=\;\frac{1}{N}\sum_{n=1}^{N}g_{r,\ell,n}.
\]
Summing over roles and layers gives the total
\[
V^{\mathrm{mb}}_{\mathrm{total}}
\;:=\;
\sum_{r=1}^{R}\sum_{\ell=1}^{L}V^{\mathrm{mb}}_{r,\ell},
\]
which is the empirical estimate of $\mathbb{E}\|U-\mathbb{E}U\|_2^2$ in the notation of Proposition~\ref{prop:variance_contraction}.

\paragraph{Depth variance.}
The depth variance measures the across-layer dispersion of the gradient within each role, evaluated per microbatch and averaged over microbatches:
\[
V^{\mathrm{depth}}_{r}
\;:=\;
\frac{1}{N}\sum_{n=1}^{N}\frac{1}{L}\sum_{\ell=1}^{L}\bigl\|g_{r,\ell,n}-\bar g_{r,\cdot,n}\bigr\|_2^2,
\qquad
\bar g_{r,\cdot,n}\;:=\;\frac{1}{L}\sum_{\ell=1}^{L}g_{r,\ell,n}.
\]
The total is $V^{\mathrm{depth}}_{\mathrm{total}}:=\sum_{r}V^{\mathrm{depth}}_{r}$, and is the empirical estimate of the quantity from Proposition~\ref{prop:Du_contraction}.

\paragraph{Observations.}
Figures~\ref{fig:microbatch_var_curves} and~\ref{fig:depth_var_curves} report $V^{\mathrm{mb}}_{\mathrm{total}}$ and $V^{\mathrm{depth}}_{\mathrm{total}}$ over training for the three runs. After the initial warmup transient (steps $\lesssim 1000$), both smoothing runs consistently exhibit lower gradient variance than baseline, in both metrics. The reduction is stable across the bulk of training and widens further in late training, consistent with Proposition~\ref{prop:variance_contraction} acting at every step along the smoothed trajectory.


\begin{figure}[h!]
    \centering
    \includegraphics[width=0.5\linewidth]{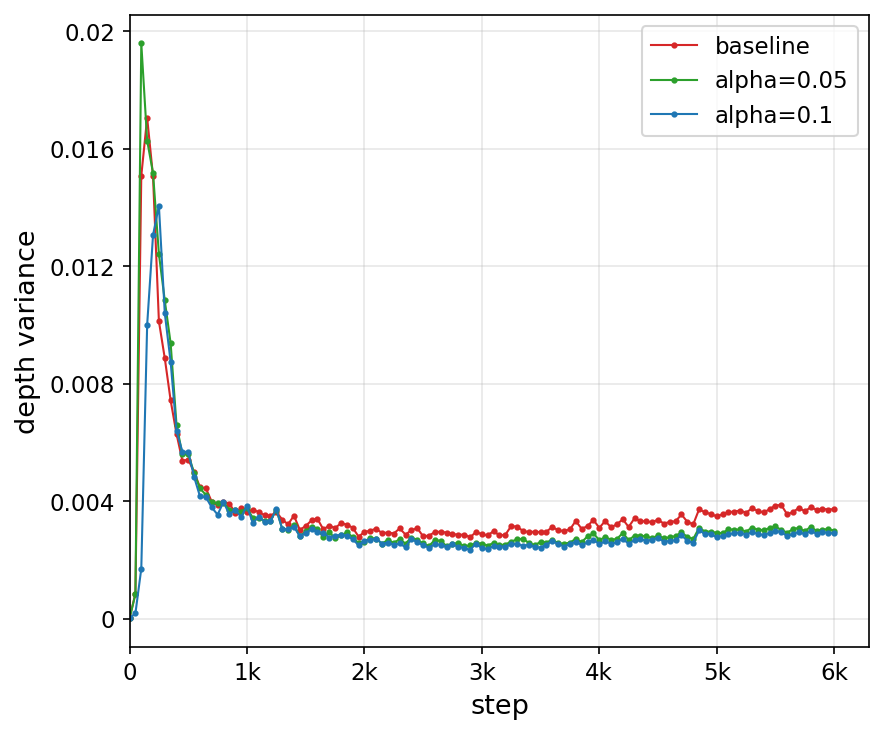}
    \caption{{\bf Lower depth gradient variance with smoothing.} Total depth variance $V^{\mathrm{depth}}_{\mathrm{total}}$ during \texttt{nanochat} d24 pretraining (7000 optimization steps, logged every 50 steps), comparing the baseline (blue) against window smoothing with $\alpha = 0.05$ (red) and $\alpha = 0.1$ (green). Both smoothing runs show consistently lower depth-wise gradient dispersion than baseline once training leaves the warmup phase, with $\alpha = 0.05$ and $\alpha = 0.1$ producing comparable reductions.}
    \label{fig:depth_var_curves}
\end{figure}

\end{document}